\documentclass[accepted]{uai2026} 
                        
\usepackage{xcolor}

\usepackage[american]{babel}

\usepackage{natbib} 
\usepackage{mathtools} 
\usepackage{booktabs} 
\usepackage{tikz} 
\usepackage{tikz} 
\usepackage{pgfplots} 
\pgfplotsset{compat=1.18}

\newcommand{\indep}{\mathrel{\perp\!\!\!\perp}}
\usepackage{makecell}
\usepackage{amsmath,amssymb,amsthm}
\usepackage{graphicx}
\usepackage[mathscr]{euscript}

\DeclareMathOperator*{\argmax}{arg\,max}
\usepackage{bm}

\usepackage{thmtools}
\usepackage{algorithm}
\usepackage{algorithmic}

\usepackage{comment}

\usepackage{amsmath,amssymb,amsthm}

\newtheorem{lemma}{Lemma}

\newtheorem{assumption}{Assumption}
\newtheorem{definition}{Definition}

\usepackage{xcolor}

\title{Fixed-Confidence Best-Arm Identification for Causal Mediation Analysis}

\author[1]{\href{mailto:<harsh.vardhan.shri@gmail.com>?Subject=Your UAI 2026 paper}{Harsh Shrivastava}{}}
\author[1]{Yuta Kawakami}
\author[1,2]{Junpei Komiyama}
\author[1]{Jin Tian}
\affil[1]{%
    Mohamed bin Zayed University of Artificial Intelligence (MBZUAI)\\
    Abu Dhabi, UAE
}
\affil[2]{%
    RIKEN AIP,
    Tokyo, Japan
}
  
\begin{document}
\maketitle

\begin{abstract}
This paper studies the problem of identifying the treatment that maximizes the expected natural direct potential outcome (NDPO), which captures the potential outcome of an intervention while excluding the pathway transmitted through a mediator that researchers may wish to remove from evaluation. We first establish population-level identification of the expected NDPO in a causal bandit setting using observable interventional distributions. We then develop a fixed-confidence best-arm identification (BAI) algorithm based on the Track-and-Stop (TaS) framework, employing a cutting-set method to solve the resulting semi-infinite optimization problem. The proposed algorithm achieves sample-efficient identification with a high-probability correctness guarantee. We prove that it satisfies $\delta$-correctness and asymptotic optimality. Finally, we validate the approach through empirical evaluations on a large-scale real-world advertising dataset (IPinYou).

\end{abstract}



\section{Introduction}\label{sec:intro}


Causal mediation analysis studies how a treatment ($X$) influences an outcome ($Y$) through different pathways
\citep{Sobel1982,Baron1986,Robins1992IdentifiabilityAE,Avin2005,Pearl09,Imai2010}.
To elucidate causal mechanisms in nonlinear models,
\citet{Pearl2001} introduced \emph{path-specific effects}, including 
the expected \emph{natural direct effect} (NDE) 
via a mediator ($Z$), defined as $\mathbb{E}[Y_{x,Z_{x_0}}] - \mathbb{E}[Y_{x_0}]$, where $Y_{x, Z_{x_0}}$ denotes the potential outcome under treatment level $x$ with the mediator set to the value it would attain under a reference level $x_0$, and $Y_{x_0}$ denotes the potential outcome under treatment level $x_0$.
The expected NDE measures the causal effect of a treatment on an outcome that is not mediated through the mediator.

In causal bandit problems, the conventional objective is to maximize the interventional mean under arm (treatment) $X=x$, i.e., $\mathbb{E}[Y_x] = \mathbb{E}[Y | do(X = x)]$ \citep{Lattimore2016, Lee2018, Lee2019}. 
However, the interventional mean aggregates effects along all causal pathways, including mechanisms that the researcher may wish to exclude from evaluation (e.g., unethical factor or positioning bias in advertisements). 
To isolate the influence via the direct pathway, this paper studies the problem of identifying the arm that maximizes the expected \emph{natural direct potential outcome} (NDPO), i.e., $\mathbb{E}[Y_{x, Z_{x_0}}]$, in a causal bandit setting. 
The NDPO corresponds to the first term of the NDE and is defined via a nested counterfactual.

In a causal bandit setting, actions correspond to interventions on $X$ \citep{Lattimore2016, Lee2018}, and one observes i.i.d. samples from the interventional distribution $\mathbb{P}(Z,Y | do(X = x))$.
The identification results of \citet{Pearl2001} are not directly formulated for this setting. 
Instead, \citet{Pearl2001} established the identification using $\mathbb{P}(Y_{x,z})$ and $\mathbb{P}(Z_{x_0})$.
Therefore, we first establish the corresponding identification assumptions of the expected NDPO at the population level using the interventional distribution $\mathbb{P}(Z,Y|do(X = x))$.

Moreover, in the bandit setting, a further objective is to identify the target arm using as few samples as possible while ensuring correctness with high probability.
Such a problem
has been extensively studied in the fixed-confidence best-arm identification (BAI) framework for multi-armed bandits \citep{mannor&shie, glynnjuneja2004, audibert2010best, Jamieson&Novak2014, garivier2016optimalbestarmidentification}.

We subsequently investigate the fixed-confidence BAI problem for identifying the expected NDPO-optimal arm, which maximizes the expected NDPO, leveraging the established population-level identification results.
To develop a BAI algorithm for identifying the expected NDPO-optimal arm, we first derive a lower bound on the sample complexity required to identify the arm that maximizes the expected NDPO. 
Compared with the standard BAI framework \citep{garivier2016optimalbestarmidentification}, the BAI for the expected NDPO-optimal arm is more challenging because evaluating the expected NDPO for a given arm $x$ requires combining information from the outcome distribution under $x$ conditional on $Z = z$, $\mathbb{P}(Z,Y | do(X = x))$, with the mediator distribution under the different baseline arm $x_0$, $\mathbb{P}(Z | do(X = x_0))$.

We then develop a new BAI algorithm (TaS-NDPO) for identifying the expected NDPO-optimal arm based on the Track-and-Stop (TaS) framework \citep{garivier2016optimalbestarmidentification}, using a cutting-set method \citep{mutapcic2009cutting} to solve the resulting semi-infinite optimization problem.
We prove that it satisfies the desired properties of $\delta$-correctness and asymptotic optimality. 
$\delta$-correctness guarantees that the algorithm identifies the optimal-arm with probability at least $1-\delta$ and asymptotic optimality refers to achieving the lower bound of sample complexity in the limit as $\delta \to 0$.

Finally, we conduct empirical evaluations on a large-scale real-world advertising dataset (IPinYou) to demonstrate the effectiveness of the proposed method.




\section{Notation and Background}\label{sec:problem_setup}

In this section, we introduce notation and background material.
Uppercase letters (e.g., $X,Z,Y$) denote random variables, and lowercase letters  (e.g., $x,z,y$) denote their realizations.
Data-dependent quantities such as the stopping time $\tau_\delta$ and the recommendation $\hat x(\tau_\delta)$ are random variables.
Let $\mathbf{1}\{\cdot\}$ denote the indicator function.
For two probability distributions $P$ and $Q$ over $\mathcal{A}$, the KL divergence is $\mathrm{KL}(P \Vert Q):=\sum_{a \in \mathcal{A}} P(a) \log \frac{P(a)}{Q(a)}$.
An event happens almost surely (a.s.) if it occurs with probability 1.

{
{\bf Structural Causal Models.}
We use the language of Structural Causal Models (SCM) as our basic semantic and inferential framework \citep{Pearl09}.
An SCM $\mathcal{M}$ is a tuple $\left<{\boldsymbol V},{\boldsymbol U}, \mathcal{F}, \mathbb{P}_{\boldsymbol U} \right>$, where ${\boldsymbol U}$ is a set of exogenous (unobserved) variables following a joint distribution $\mathbb{P}_{\boldsymbol U}$, and ${\boldsymbol V}$ is a set of endogenous (observable) variables whose values are determined by structural functions $\mathcal{F}=\{f_{V_i}\}_{V_i \in {\boldsymbol V}}$ such that $v_i:= f_{V_i}({\mathbf{pa}}_{V_i},{\boldsymbol u}_{V_i})$ where ${\mathbf{PA}}_{V_i} \subseteq {\boldsymbol V}$ and $U_{V_i} \subseteq {\boldsymbol U}$. 
Each SCM $\mathcal{M}$ induces an observational distribution $\mathbb{P}_{\boldsymbol V}$ over ${\boldsymbol V}$, and a causal graph $G(\mathcal{M})$ over ${\boldsymbol V}$ in which there exists a directed edge from every variable in ${\mathbf{PA}}_{V_i}$ to $V_i$.
An intervention of setting a set of endogenous variables ${\boldsymbol X}$ to constants ${\boldsymbol x}$, denoted by $do({\boldsymbol x})$, replaces the original equations of ${\boldsymbol X}$ by the constants ${\boldsymbol x}$ and induces a \textit{sub-model}  $\mathcal{M}_{{\boldsymbol x}}$.
We denote the potential outcome $Y$ under intervention $do({\boldsymbol x})$ by $Y_{{\boldsymbol x}}({\boldsymbol u})$, which is the solution of $Y$ in the sub-model $\mathcal{M}_{{\boldsymbol x}}$ given ${\boldsymbol U}={\boldsymbol u}$. 
}

{
{\bf Causal Mediation Analysis.}
Let $X$ be a 
treatment variable (arm), $Y$ be an 
outcome, and $Z$ be a
mediator variable. 
The following nonparametric SCM ($\mathcal{M}$) is used in causal mediation analysis.
\begin{equation}
\label{scm}
\begin{gathered}
  \mathcal{M}:X:=f_X(U_X),
    Z:=f_{Z}(X,U_{Z}),\\
    Y:=f_Y(X,Z,U_Y),
   \end{gathered}
\end{equation}
where $U_X$, $U_{Z}$, and $U_Y$ are latent exogenous variables,
with bidirected edges indicating unmeasured confounders affecting the variables.
Figure \ref{fig:motivating_causal_graph} shows a causal graph representing $\mathcal{M}$.
}
Under SCM $\mathcal{M}$,
\citet{Pearl2001} defined the unit-level 
NDE
given a reference value $x_0$ as $Y_{x,Z_{x_0}}(\bm{u})-Y_{x_0}(\bm{u})$ for each subject $\bm{u}$, and the expected NDE given a reference value $x_0$ as
$\mathbb{E}[Y_{x,Z_{x_0}}] - \mathbb{E}[Y_{x_0}]$.
\citep{Pearl2001} provided identification conditions of the expected NDE, that is $Y_{x,z}\indep Z_{x_0}$ for any $x$ and $z$.


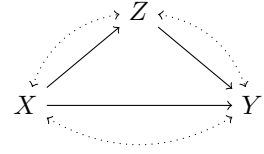
\begin{figure}[tb]
    \centering
\scalebox{1}{
\begin{tikzpicture}
\node (z) at (0,1.25) {$Z$};
\node (y) at (1.5,0) {$Y$};
\node (x) at (-1.5,0) {$X$};

\path (x) edge[->] (y);
    
\path (z) edge[->] (y);
\path (x) edge[->] (z);



\path (x) edge[<->,dotted,bend left] (z);
\path (z) edge[<->,dotted,bend left] (y);
\path (x) edge[<->,dotted,bend right] (y);
    

\end{tikzpicture}
} 
    \caption{A causal graph representing $\mathcal{M}$.}
    \label{fig:motivating_causal_graph}
\end{figure}


Throughout this paper, we consider SCM $\mathcal{M}$ with a discrete arm (treatment) variable
$X \in \{0,1,\dots,K-1\} =: \mathcal X$, and a discrete mediator
$Z \in \{0,1,\dots,M-1\} =: \mathcal Z$, and a binary outcome $Y \in \{0,1\}=: \mathcal Y$.
We fix a baseline arm $x_0 \in \mathcal X$.
The choice of $x_0$ is predetermined based on domain knowledge in practice.




{\bf Best-arm identification.}
In the fixed-confidence best-arm identification (BAI) setting,
a learner sequentially samples arms and aims to identify the optimal-arm
with probability at least $1-\delta$ for a prescribed confidence level
$\delta\in(0,1)$.
A sequential algorithm consists of a sampling rule, a stopping time
$\tau_\delta$, and a recommendation rule $\hat x(\tau_\delta)$.
An algorithm is said to be \emph{$\delta$-correct} if
$\mathbb P\!\left(\tau_\delta<\infty,
\hat x(\tau_\delta) \neq x^\star
\right)
\le \delta$,
where $x^\star$ denotes the optimal-arm under the true model.
The objective is to design a $\delta$-correct algorithm that minimizes
the expected stopping time $\mathbb E[\tau_\delta]$.

\section{Research Objective}


In this paper, we focus exclusively on the first term of the NDE, which we call the natural direct potential outcome (NDPO). 
\begin{definition}[NDPO]
For each subject $\bm{u}$,
we define the unit-level natural direct potential outcome (NDPO) at a reference $x_0$ as $Y_{x,Z_{x_0}}(\bm{u})$.
Its population-level expectation is defined as $\mathbb{E}[Y_{x,Z_{x_0}}]:=\mathbb{E}_{\bm{U}}[Y_{x,Z_{x_0}}(\bm{U})]$.
\end{definition}
We do not study the contrast with the baseline $Y_{x_0}$.
The expectation of NDPO
means:
``\emph{the expected outcome if we set the arm $X$ to $x$, while keeping the mediator $Z$ at the level it would have taken under the reference arm $x_0$.}"
NDPO is a nested counterfactual quantity that measures the influence of the intervention $X = x$, removing the effect transmitted through $Z$.

In causal bandit problems, the conventional objective is to maximize the interventional mean  $\mathbb{E}[Y_x]=\mathbb{E}[Y|do(X=x)]$ \citep{Lattimore2016, Lee2018, Lee2019}.
This paper studies the maximization of the expectation of the NDPO.
We denote the optimal-arm w.r.t. the expected NDPO under SCM ${\mathcal{M}}$ as follows:
\begin{equation}
x^\star:= \arg\max_{x\in\mathcal X} \mathbb{E}[Y_{x,Z_{x_0}}].
\end{equation}
The optimal-arm w.r.t. the expected NDPO is the arm that maximizes the expectation of NDPO.
We present motivating examples illustrating why we study the maximization of the expectation of NDPO in Appendix \ref{app-moti}.

\section{Bridging Causal Inference and Best Arm Identification}

{Population-level identification of causal quantities has been extensively studied within the causal inference framework \citep{Pearl09}. 
In this section, we establish the population-level identification of the expected NDPO using interventional distribution $\mathbb{P}(Z,Y|do(X=x))$ for a discrete $Y$ and a connection between the causal inference framework and best arm identification (BAI) theory \citep{garivier2016optimalbestarmidentification}.}
All proofs are provided in Appendix~\ref{appa}.

{\bf Population-level causal identification of the expected NDPO under the causal bandit setting.}
{We consider a causal bandit setting in which researchers sequentially collect samples by pulling arms in order to identify the optimal-arm w.r.t. NDPO.}
{Pulling arm $x$ corresponds to performing the intervention $do(X=x)$ \citep{Lattimore2016}.
In our mediation setting, pulling arm $x$ yields i.i.d. samples from the interventional distribution $\mathbb{P}(Z,Y|do(X=x))$.
}

{To identify $\mathbb{E}[Y_{x,Z_{x_0}}]$ from $\mathbb{P}(Z,Y|do(X=x))$, we impose the following assumption on potential outcomes.}


\begin{assumption}[Assumption for Identification for NDPO]
\label{idenNDE}
$Y_{x,z}\indep Z_{x'}$ for all $x,x'\in\mathcal X$ and $z\in\mathcal Z$.
\end{assumption}
Here, $x'$ denotes an arbitrary reference level. 
\citet{Pearl2001} considers the reference-specific condition (which we call Assumption 1' here):
``$Y_{x,z}\indep Z_{x_0}$ for all $x\in\mathcal X$ and $z\in\mathcal Z$ for a fixed $x_0$.''
Formally, Assumption 1 appears stronger, as it requires independence w.r.t. ($Z_{x'}$) for every ($x'$), rather than a single reference level ($x_0$). 
However, in practice, such reference-specific assumptions are typically understood to hold irrespective of the particular choice of baseline level.
That is, Assumption 1' is implicitly required to hold for any admissible $x_0$.
More importantly, \citet{Pearl2001} provides a graphical characterization of Assumption 1':
\begin{equation}
Y \perp Z \quad \text{in } G_{\underline X,\underline Z},
\end{equation}
that is, $Y$ is d-separated from $Z$ in the graph obtained by removing all outgoing edges from $X$ and $Z$. 
This condition is purely graphical and therefore does not depend on the specific value of $x'$. 
Consequently, under the graphical interpretation, Assumptions 1 and 1' are equivalent, as they correspond to the same d-separation condition. 
Neither Assumption 1 nor Assumption 1' is testable from observed data alone. 
In applications, what can be assessed is the plausibility of the underlying causal graph. 
Both assumptions basically correspond to the absence of unmeasured confounding between $Z$ and $Y$.


Then, we have the following identification theorem for a discrete $Y$:
\begin{restatable}{theorem}{Theoremone}
\label{thm:nde_identification}
Under SCM $\mathcal M$ and Assumption~\ref{idenNDE},
the expected NDPO is identified
by $\theta(x)$, where
\begin{align}
\theta(x)
&= \sum_{y,z} yP_x(y|z)P_{x_0}(z),
\label{eq:theta_id}
\end{align}
where $P_x(y|z)=\mathbb{P}(Y=y | Z=z,do(X=x))$ and $P_x(z)=\mathbb{P}(Z=z | do(X=x))$.
\end{restatable}


This differs in form from Theorem 1 of \citet{Pearl2001}: we express $\theta(x)$ using the joint interventional distribution $\mathbb{P}(Z,Y|do(X=x))$, rather than using $\mathbb{P}(Y_{x,z})$ and $\mathbb{P}(Z_{x_0})$ \citep{Pearl2001}.
The equality $\mathbb{P}(Y_{x,z}) = \mathbb{P}(Y|do(X=x),do(Z=z))$ corresponds to intervening on both $X$ and $Z$, which does not match our causal bandit setting.

{\bf Connection to BAI theory.}
Estimation of mediated effects from a fixed dataset has been widely studied in the causal inference literature.
For example, $\theta(x)$ can be directly estimated using plug-in empirical estimators of $P_x(y|z)$ and  $P_x(z)$ \citep{Imai2010b}. 
Moreover, \citet{VanderWeele2010,Tchetgen2012} used parametric or semiparametric regression models to estimate mediated effects.

In contrast, under a bandit setting, data are collected sequentially.
In this setting, the researcher is motivated to identify the target using as few samples as possible while guaranteeing correctness with high probability \citep{garivier2016optimalbestarmidentification}.
Accordingly, our goal is to develop a BAI algorithm that identifies the arm maximizing the NDPO through $\theta(x)$ with minimal sample complexity, subject to a high-probability correctness guarantee.

We next investigate the BAI problem of identifying the arm that maximizes $\theta$, apart from the causal inference perspective, rather than directly optimizing the expected NDPO.

\section{Sample Complexity Lower Bound}
\label{sec:lower_bound}

We derive a lower bound on the expected stopping time of any $\delta$-correct algorithm that identifies the arm maximizing $\theta$, analogous to the classical fixed-confidence BAI lower bounds for a binary outcome $Y\in\{0,1\}$ (for simplicity) \citep{lairobbins1985, garivier2016optimalbestarmidentification}.
This reveals the fundamental limitations of the BAI framework.

To study BAI for $\theta$, we introduce the following notation.
Let $\mathcal P$ denote the set of all $P$, where $P$
is a collection of interventional distributions
$(P_x)_{x\in\mathcal X}$,
where $P_x$ denotes $\mathbb{P}(Z,Y|do(X=x))$.
For each $P \in \mathcal P$, we define
\begin{align}
\label{eq:theta_P_def}
&\theta_P(x):=
\sum_{z\in\mathcal Z}
P_x(1| z)\,
P_{x_0}(z),\\
&x^\star(P):=
\arg\max_{x\in\mathcal X}
\theta_P(x).
\end{align}
We impose the following assumption.
\begin{assumption}[Unique optimum]
\label{ass:unique_opt}
For the true model $P \in \mathcal P$,
the maximizer $x^\star(P)$ is unique.
\end{assumption}

Let $\tau_\delta$ denote the stopping time at confidence level $\delta$, and let $\hat x_{\tau_\delta}$ be the recommended arm.
We define a $\delta$-correct algorithm for the optimal-arm w.r.t. $\theta$ as follows:
\begin{definition}[$\delta$-correct algorithm]
Let $\delta\in(0,1)$.
An algorithm for identifying the optimal-arm w.r.t. $\theta$ is $\delta$-correct if, for every
$P \in \mathcal P$,
\begin{equation}
\mathbb P_P\!\left(\tau_\delta<\infty, 
\hat x_{\tau_\delta}
\neq
x^\star(P)
\right)
\le \delta.
\end{equation}
\end{definition}

To derive the lower bound, we consider an alternative model in which the optimal-arm differs from that in the underlying model.
For a fixed model $P \in \mathcal P$, define the set of alternatives:
$\mathrm{Alt}(P):=\left\{
Q \in \mathcal P :
x^\star(Q) \neq x^\star(P)
\right\}$.
Equivalently, $Q \in \mathrm{Alt}(P)$ if there exists
$x \neq x^\star(P)$ such that
\begin{equation}
\label{eq:alt_P}
\theta_Q(x)
\ge
\theta_Q(x^\star(P)).
\end{equation}

For two models $P,Q \in \mathcal P$ and an arm $x\in\mathcal X$, define the
interventional joint distribution
\begin{equation}
P_x(z,y)
:=
\mathbb{P}_P(Z=z, Y=y | do(X=x)),
\end{equation}
and similarly $Q_x(z,y)$ for model $Q$.
We denote the corresponding marginals and conditionals by
\begin{equation}
P_x(z)
:=
\sum_{y\in\mathcal Y} P_x(z,y),
P_x(y| z)
:=
\frac{P_x(z,y)}{P_x(z)},
\end{equation}
and similarly for $Q_x(z)$ and $Q_x(y| z)$.

Since $P_x(z,y)=P_x(z)P_x(y| z)$, the KL divergence decomposes as
\begin{align}
\label{eq:kl_decomp}
&\mathrm{KL}(P_x \Vert Q_x)
=
\mathrm{KL}\big(P_x(Z)\Vert Q_x(Z)\big)
\nonumber\\
&+
\sum_{z\in\mathcal Z}
P_x(z)
\,
\mathrm{KL}
\big(
P_x(Y| z)
\Vert
Q_x(Y| z)
\big).
\end{align}

The KL decomposition in \eqref{eq:kl_decomp} splits
$\mathrm{KL}(P_x\Vert Q_x)$ into a mediator term and a conditional outcome term weighted by $P_x(z)$.
Hence discrepancies in $P_x(1| z)$ at rare mediator values contribute little to the total divergence,
so distinguishing such differences requires more samples from arm $x$.

Then, we have the following theorem:
\begin{restatable}{theorem}{Theoremtwo}[Instance-dependent lower bound]
\label{thm:nde_lower_bound}
Let $\delta\in(0,1)$.
Under Assumption~\ref{ass:unique_opt},
for any $\delta$-correct algorithm and any
$P \in \mathcal P$,
\begin{equation}
\mathbb E_P[\tau_\delta]
\;\ge\;
T^\star(P)\,
\mathrm{kl}(\delta,1-\delta),
\end{equation}
where
$\mathrm{kl}(\delta,1-\delta)
=
\delta \log\!\frac{\delta}{1-\delta}
+
(1-\delta)\log\!\frac{1-\delta}{\delta}$
and
\begin{equation}
\label{eq:char_time}
(T^\star(P))^{-1}
=
\sup_{w\in\Sigma_K}
\inf_{Q\in\mathrm{Alt}(P)}
\sum_{x\in\mathcal X}
w_x
\,
\mathrm{KL}(P_x\Vert Q_x)
\end{equation}
with
$\Sigma_K:=\{
w\in\mathbb R_+^K:
\sum_{x\in\mathcal X} w_x = 1
\}$.
\end{restatable}
Here $\mathbb R_+^K:=\{v\in\mathbb R^K: v_x\ge 0, \forall x\}$.
The vector $w\in\Sigma_K$ represents asymptotic sampling proportions: $w_x$ is the fraction of samples allocated to arm $x$ in the long run. The $\mathrm{kl}(\cdot,\cdot)$ denotes the binary relative entropy.
The sup--inf form in~\eqref{eq:char_time} chooses proportions $w$ that maximize the
worst-case information rate for distinguishing the true model $P$ from alternatives
under which the identity of the optimal-arm differs.




{\bf Technical difficulties in BAI for $\theta$.}
The lower bound in Theorem~\ref{thm:nde_lower_bound} highlights two essential differences from classical best-arm identification.
In classical BAI, the objective is
$\mu_x = \mathbb E_P[Y| do(X=x)]$,
and the characteristic time reduces to
$(T^\star_{\mathrm{BAI}}(P))^{-1}
=
\sup_{w\in\Sigma_K}
\inf_{x\neq x^\star(P)}
\sum_{a\in\mathcal X}
w_a
\mathrm{KL}
(
P_a
\Vert
Q^{(x)}_a
)$,
where $Q^{(x)}$ denotes an alternative model that coincides with $P$
on all arms except $x^\star(P)$ and $x$, and is chosen so that
$x$ becomes optimal under $Q^{(x)}$
\citep{garivier2016optimalbestarmidentification}.
This yields a low-dimensional convex program in which, for each suboptimal-arm
$x\neq x^\star(P)$ (called a \emph{competitor}), the alternative model
modifies only the true best arm and the single competing arm $x$.

In contrast, $\theta$ couples outcome behavior under arm $x$
with the mediator distribution under the baseline arm $x_0$.
The constraint \eqref{eq:alt_P}
depends jointly on mediator and outcome mechanisms,
so the inner minimization cannot be reduced to modifying only the best arm and a
single competitor while keeping all other arms fixed, as in classical BAI.
Instead, the baseline mediator distribution $Q_{x_0}(Z)$ enters the objective
for every arm through $\theta_Q(\cdot)$, creating an explicit coupling across arms.
Moreover, the KL decomposition
\eqref{eq:kl_decomp}
shows that deviations in
$P_x(y| z)$
are weighted by the mediator probability $P_x(z)$.
If decisive differences occur at mediator values with small probability mass,
the effective information rate is reduced,
increasing the characteristic time $T^\star(P)$.

These two features—the non-separable alternative constraint and the information that accumulates at the level of mediator–outcome
cells $(x,z)$
—do not arise in classical BAI.
We call each pair $(x,z)\in\mathcal X\times\mathcal Z$ a \emph{cell}, corresponding
to the conditional outcome parameter $P_x(y| z)$ 
and the frequency with which
mediator value $z$ occurs under arm $x$.

\section{A Track-and-Stop Algorithm for Causal Mediation Analysis}
\label{sec:algorithm}

We adapt the Track-and-Stop (TaS) framework for fixed-confidence BAI
\citep{garivier2016optimalbestarmidentification} to the objective $\theta$.
TaS combines forced exploration, a plug-in solution of the characteristic
max--min allocation problem, and a generalized likelihood ratio stopping rule.
However, the structure of the lower bound in
Section~\ref{sec:lower_bound} prevents a direct application of the classical
arm-local TaS analysis. Two modifications are required:

\textbf{(i) Cell-level information control.}
By the KL decomposition~\eqref{eq:kl_decomp}, information accumulates at the
mediator–outcome cell level $(x,z)$.
Arm-level forced exploration alone does not ensure that all cell counts
$N_{x,z}(t)$ diverge, which is necessary for consistent estimation of
$P_x(y| z)$
and evaluation of the likelihood-ratio statistic used in the stopping rule.

\textbf{(ii) Non-separable alternative optimization.}
The alternative constraint defining $\mathrm{Alt}(P)$ couples arms through the
baseline mediator distribution $P_{x_0}(Z)$.
As a result, the inner minimization in the characteristic time cannot be
reduced to a single best-arm–competitor pair, and instead yields a
bi-convex projection problem.

By modifying these components, we develop a new algorithm, TaS-NDPO, which is presented in Algorithm \ref{alg:tas_ndpo}. 
The following subsections detail the resulting sampling rule and the associated optimization procedure.


\subsection{Plug-in Estimates}
\label{subsec:empirical_model}


We first define the empirical model and the corresponding plug-in estimators of $\theta(x)$.
At round $s=1,2,\dots$, the learner selects $X_s\in\mathcal X$
and observes $(Z_s,Y_s)\sim P_{X_s}$.
Define the arm and cell counts at the round $t$
\begin{align}
&N_x(t):=\sum\nolimits_{s=1}^t \mathbf{1}\{X_s=x\},\\
&N_{x,z}(t):=\sum\nolimits_{s=1}^t \mathbf{1}\{X_s=x,Z_s=z\},\\
&S_{x,z}(t):=\sum\nolimits_{s=1}^t \mathbf{1}\{X_s=x,Z_s=z,Y_s=1\}.
\end{align}
Then, the empirical
distributions are
\begin{equation}
\hat P_x(z;t):=\frac{N_{x,z}(t)}{N_x(t)},
\hat P_x(1| z;t):=\frac{S_{x,z}(t)}{N_{x,z}(t)}.
\end{equation}




The empirical joint model at time $t$ is the collection
$\hat P(t):=(\hat P_x(\,\cdot\,;\,t))_{x\in\mathcal X}$,
where for each arm $x\in\mathcal X$, $\hat P_x(\,\cdot\,;\,t)$ denotes the empirical
joint distribution of $(Z,Y)$ under arm $x$, with probability mass function
$\hat P_x(z,y;\,t):=\hat P_x(z;\,t)\,\hat P_x(y| z;\,t)$ for $(z,y)\in\mathcal Z\times\{0,1\}$.
The plug-in estimate of $\theta$ is
\begin{equation}
\hat\theta_t(x):=
\sum_{z\in\mathcal Z}
\hat P_x(1| z;t)\,
\hat P_{x_0}(z;t),
\end{equation}
and the empirical best arm is the smallest index in the argmax set
\begin{equation}
\hat x(t)\in\arg\max_{x\in\mathcal X}\hat\theta_t(x).
\end{equation}


\subsection{Computation of the Plug-in Allocation}
\label{subsec:optimization_details}

In contrast to classical BAI, the alternative constraint couples mediator
and outcome mechanisms, yielding a non-separable KL projection with bilinear
structure of Eq.~\eqref{eq:char_time}. 
In this section, after reducing the dimension of the optimization, we introduce a cutting-set method that solves this optimization. 

{\bf Plug-in allocation.}
Replacing the true model $P$ by $\hat P(t)$ in the characteristic-time
optimization, 
we define the alternative set
$\mathrm{Alt}(\hat P(t)):=\{Q\in\mathcal P : x^\star(Q)\neq \hat x(t)\}$
and the target allocation
\begin{equation}
    \label{eq:plugin_allocation}
    \hat w(t)\in\argmax_{w\in\Sigma_K}
    \inf_{Q\in\mathrm{Alt}(\hat P(t))}
    \sum_{x\in\mathcal X}
    w_x\,\mathrm{KL}(\hat P_x(t)\Vert Q_x).
\end{equation}

{\bf Decomposition over competitors.}
Let $P\in\mathcal P$ admit a unique maximizer $x^\star(P)$.
The alternative set is given as
\begin{equation}
\mathrm{Alt}(P):=
\bigcup_{x'\neq x^\star(P)}
\{
Q\in\mathcal P :
\theta_Q(x') \ge \theta_Q(x^\star(P))
\}.
\end{equation}
Accordingly, the inner infimum in
\eqref{eq:plugin_allocation}
decomposes into $K-1$ subproblems indexed by competitors
$x'\neq x^\star(P)$.

{\bf Reduction to three arms.}
Fix a competitor $x'\neq x^\star(P)$ and consider
\begin{equation}
\label{eq:inner_full}
\inf_{Q\in\mathcal P}
\sum_{a\in\mathcal X}
w_a\,\mathrm{KL}(P_a\Vert Q_a)
\text{ s.t. }
\theta_Q(x') \ge \theta_Q(x^\star(P)).
\end{equation}

\begin{restatable}{proposition}{Propositionone}[Reduction to three arms]
\label{prop:three_arm_reduction}
For the inner minimization oracle in \eqref{eq:inner_full}, it is without
loss of optimality to restrict attention to solutions satisfying
$Q_a = P_a$ for all $a\notin\{x_0,x^\star(P),x'\}$.
\end{restatable}

Hence, the inner minimization effectively involves only three arms:
$x_0$, $x^\star(P)$, and a single competitor $x'$.

{\bf Reduced inner problem.}
Let $x^\star = x^\star(P)$ and fix $x'\neq x^\star$.
After restriction, the free variables are
$q_0(z) := Q_{x_0}(z),
r^\star(z) := Q_{x^\star}(1| z),
r'(z) := Q_{x'}(1| z), z\in\mathcal Z$.
The feasible set is
$\Delta_{|\mathcal Z|}
\times
[0,1]^{|\mathcal Z|}
\times
[0,1]^{|\mathcal Z|}$,
where
$\Delta_{|\mathcal Z|}
:=
\{
q\in\mathbb R_+^{|\mathcal Z|}
:
\sum_{z\in\mathcal Z} q(z)=1
\}$
is the probability simplex over $\mathcal Z$.
This set is compact and convex.
The objective becomes
\begin{equation}
\begin{aligned}
\label{eq:reduced_objective}
& w_{x_0}
\sum\nolimits_{z}
\mathrm{KL}(
\hat P_{x_0}(z)\,\Vert\,q_0(z)
)
\\
&+
w_{x^\star}
\sum\nolimits_{z}
\hat P_{x^\star}(z)
\mathrm{KL}(
\hat P_{x^\star}(1| z)
\Vert
r^\star(z)
)
\\
&+
w_{x'}
\sum\nolimits_{z}
\hat P_{x'}(z)
\mathrm{KL}(
\hat P_{x'}(1| z)
\Vert
r'(z)
).
\end{aligned}
\end{equation}


The original constraint is
\begin{equation}
\sum\nolimits_z r'(z)q_0(z)\ge\sum\nolimits_z r^\star(z)q_0(z).
\end{equation}

This inequality must hold with equality at an optimal solution. Indeed, if
there were an optimal solution for which the inequality were strict, then
there would exist a sufficiently small $\alpha>0$ such that replacing
\begin{equation}
\begin{aligned}
&r^\star(z)\leftarrow(1-\alpha)r^\star(z)+\alpha \hat P_{x^\star}(1|z),\\
&r'(z)\leftarrow(1-\alpha)r'(z)+\alpha \hat P_{x'}(1|z).
\end{aligned}
\end{equation}

reduces the objective value in \eqref{eq:reduced_objective}, while not
changing the sign of the difference between the left-hand side and
right-hand side of the constraint. Therefore, no optimal solution can satisfy
the constraint with strict inequality, and the constraint reduces to the
scalar bilinear equality

\begin{equation}
\label{eq:bilinear_constraint}
\sum\nolimits_{z}
r'(z)\,q_0(z)
=
\sum\nolimits_{z}
r^\star(z)\,q_0(z).
\end{equation}

\begin{restatable}{proposition}{Propositiontwo}[Bi-convex structure]
\label{prop:biconvex}
The optimization problem defined by
\eqref{eq:reduced_objective}--\eqref{eq:bilinear_constraint}
is convex in $(r^\star,r')$ for fixed $q_0$,
and convex in $q_0$ for fixed $(r^\star,r')$,
but not jointly convex.
\end{restatable}

{\bf Cutting-set method.}
Even after reducing the dimension to three, the resulting problem is still a linear semi-infinite program in $w = (w_{x_0}, w_{x^\star}, w_{x'})$. This is because the constraint involving $(r'(z), q_0(z), r^\star(z))$ must hold over a continuous domain.


The cutting-set
method \citep{mutapcic2009cutting} starts with some initial set $\mathcal Q^{(1)}\subset\mathrm{Alt}(\hat P)$.
\footnote{In our implementation, the active set is initialized with one constraint for
each competitor arm. This avoids degenerate initial allocations and ensures
that each competitor is represented in the master problem. If a candidate
allocation assigns zero mass to an arm $a$, then the competitor subproblem
with $x'=a$ is selected as a hardest constraint and added to the active set;
the subsequent master update then accounts for this arm and assigns it
positive mass when needed.}
At each iteration $s=1,2,\dots$, the method maintains
a finite active set $\mathcal Q^{(s)}\subset\mathrm{Alt}(\hat P)$ and proceeds as follows:
(i) solving
$\max_{w\in\Sigma_K}
\min_{Q\in\mathcal Q^{(s-1)}}
\sum_x w_x \mathrm{KL}(\hat P_x\Vert Q_x)$ on $w$,
which is a finite-dimensional linear programming,
and
(ii) computing a most-violated alternative model
via the reduced bi-convex problem, and add it to $\mathcal Q^{(s)}$.


{\bf Optimization properties.}
Under mild conditions, including boundedness of the objective and
empirical means bounded away from $0$ and $1$, the cutting-set method
converges to a global optimum \citep[Section~5.2]{mutapcic2009cutting} assuming that (ii) is exact.
These conditions are satisfied in our setting when the empirical estimates
remain in $[\varepsilon,1-\varepsilon]$ for some $\varepsilon>0$.

Regarding the optimality of (ii), we solve it via alternating minimization.
The implementation details are provided in Appendix~\ref{app:inner_qp}.
Under standard regularity conditions,
block coordinate descent for bi-convex problems
converges to a stationary point \citep{Tseng2001}.
However, in general, such a stationary point may not be globally optimal. 

For ease of discussion, our theoretical analysis in Section \ref{sec:analysis} assumes access to the exact optimizer.

{\bf Computation.}
In our experiments, the optimization routine is not a major computational
bottleneck: TaS-NDPO takes 22.5 seconds on average for runs with up to
100{,}000 samples; detailed runtime comparisons are reported in
Appendix~\ref{app:runtime_comparison}.

\subsection{Sampling Rule}
\label{subsec:sampling_rule}

We use the convention that after $t$ pulls we have counts $N_x(t),N_{x,z}(t)$
(Section~\ref{subsec:empirical_model}), and we choose the next arm $X_{t+1}$
based on plug-in quantities computed from data up to time $t$.
Given the plug-in allocation $\hat w(t)$, the algorithm follows a D-tracking rule
augmented with (i) competitor coverage and (ii) cell-level forcing.

{\bf D-tracking.}
When neither coverage nor forcing is active, we apply the D-tracking rule
\citep{garivier2016optimalbestarmidentification}, i.e., 
$X_{t+1}$ is updated  by the smallest index of $\arg\max_{x\in\mathcal X}\bigl(t\,\hat w_x(t)-N_x(t)\bigr)$.

{\bf Competitor coverage.}
Let $C_x(t)$ denote the number of times arm $x$ has been selected as the
competitor up to time $t$, with initialization $C_x(0)=0$ and update
\begin{equation}
C_x(t)=C_x(t-1)+\mathbf{1}\{x'(t)=x\}.
\end{equation}
Fix $a\in(0,1)$ and define $h(t):=\lceil t^a\rceil$,
where $\lceil\cdot\rceil$ denotes the ceiling function.
After observing $t$ samples, we select the next competitor $x'(t+1)$ as follows.
If $\min_{x\in\mathcal X} C_x(t)<h(t)$, choose
\begin{equation}
x'(t+1)\in \arg\min\{C_x(t): C_x(t)<h(t)\},
\end{equation}
breaking ties by the smallest index.
Otherwise, choose $x'(t+1)$ as any competitor $x'\neq \hat x(t)$ attaining the minimum among the competitor-wise inner problems described in Section~\ref{subsec:optimization_details}, breaking ties by the smallest index.
This rule guarantees that each arm is selected as a competitor at least $h(t)-1=\Omega(t^a)$ times (for all sufficiently large $t$).

{\bf Cell-level forcing.}
Define the relevant arm set at time $t+1$ by
$\mathcal A_{\mathrm{rel}}(t+1)=\{x_0,\hat x(t),x'(t+1)\}$.
Let $g(t)$ be nondecreasing with $g(t)\to\infty$ and $g(t)=o(h(t))$.
For any arm $x$, define its minimum cell count
$m_x(t):=\min_{z\in\mathcal Z} N_{x,z}(t)$.
Define the set of deficient relevant arms (computed from time-$t$ counts)
$\mathcal D(t):=\{x\in\mathcal A_{\mathrm{rel}}(t+1): m_x(t)<g(t)\}$.
If $\mathcal D(t)\neq\emptyset$, we force exploration by choosing
$X_{t+1}\in \arg\min_{u\in\mathcal D(t)}(m_u(t),\,N_u(t))$,
i.e., we minimize $m_u(t)$ and break ties by the smallest total count $N_u(t)$
(and then by the smallest index).
Otherwise, we perform D-tracking 
breaking ties by the smallest index.

{\bf Role of coverage and forcing.}
Competitor coverage ensures each arm is compared often enough, while cell-level forcing
ensures all relevant cell counts $N_{x,z}(t)$ grow, which is needed for consistency and for the likelihood-ratio statistic in the stopping rule. Both mechanisms act on a vanishing fraction of rounds since $h(t)=t^a$ with $a\in(0,1)$ and $g(t)=o(h(t))$, so the dominant sampling dynamics are governed by D-tracking.

\begin{algorithm}[!t]
\caption{Algorithm of TaS-NDPO}
\label{alg:tas_ndpo}
\begin{algorithmic}[1]
\REQUIRE Confidence $\delta\in(0,1)$; baseline arm $x_0$; $h(t)=\lceil t^a\rceil$ for some $a\in(0,1)$; nondecreasing $g(t)$ such as $g(t)\to\infty$ and $g(t)=o(h(t))$.
\STATE Pull each arm once and update counts; set $C_x \gets 0$ for all $x\in\mathcal X$.
\FOR{$t=|\mathcal X|,|\mathcal X|+1,\dots$}
    \STATE Compute $\hat P(t)$, $\hat\theta_t(\cdot)$, $\hat x(t)\in\arg\max_{x\in\mathcal X}\hat\theta_t(x)$, and $\hat w(t)$.
    \STATE Compute $Z_t(x,\hat x(t))$ for all $x\neq \hat x(t)$.
    \IF{$\min_{x\neq \hat x(t)} Z_t(x,\hat x(t)) \ge \beta(t,\delta)$}
        \RETURN $\hat x(t)$
    \ENDIF

    \STATE \textbf{Competitor coverage:}
    \IF{$\min_{x\in\mathcal X} C_x < h(t)$}
        \STATE $x'(t{+}1)$ is updated by the smallest index of $\arg\min\{C_x:\ C_x<h(t)\}$.
    \ELSE
        \STATE $x'(t{+}1)\leftarrow$ a minimizer of the competitor-wise inner problems described in Section~\ref{subsec:optimization_details}.
    \ENDIF
    \STATE $C_{x'(t+1)} \gets C_{x'(t+1)} + 1$,\\
$\mathcal A_{\mathrm{rel}} \gets \{x_0,\hat x(t),x'(t+1)\}$,\\
 $\mathcal D \gets \{u\in\mathcal A_{\mathrm{rel}}:\ \min_{z\in\mathcal Z}N_{u,z}(t)<g(t)\}$.

    \STATE \textbf{Arm selection:}
    \IF{$\mathcal D\neq\emptyset$}
        \STATE $X_{t+1}$ is updated by  the smallest index of $\arg\min_{u\in\mathcal D}(\min_{z}N_{u,z}(t),\,N_u(t)).$
    \ELSE
        \STATE $X_{t+1}$ is updated by the smallest index of $\arg\max_{x\in\mathcal X}(t\hat w_x(t)-N_x(t)).$
    \ENDIF

    \STATE Pull $X_{t+1}$; observe $(Z_{t+1},Y_{t+1})$; update counts.
\ENDFOR
\RETURN $\hat x(t)$
\end{algorithmic}
\end{algorithm}

\subsection{Stopping rule}
\label{subsec:stopping_rule}

The stopping rule should guarantee $\delta$-correctness while stopping as soon as the data exclude every model under which a suboptimal-arm could be optimal.
As in classical TaS, we use a likelihood-based test against the
closest alternative model.

{\bf Likelihood function.}
For any $Q\in\mathcal P$, define the negative log-likelihood ratio
$\mathcal L_t(Q)
:=
\sum_{a\in\mathcal X}
N_a(t)\,
\mathrm{KL}(\hat P_a(t)\Vert Q_a)$,
where $\hat P_a(t)$ denotes the empirical joint distribution of $(Z,Y)$
under arm $a$.

{\bf GLRT statistic.}
At time $t$, for each competitor $x\neq \hat x(t)$,
define the generalized likelihood ratio statistic (GLRT)
\begin{equation}
\label{eq:glrt_stat}
Z_t(x,\hat x(t))
:=
\inf_{Q\in\mathcal P:\,\theta_Q(x)\ge\theta_Q(\hat x(t))}
\mathcal L_t(Q),
\end{equation}
where
$\theta_Q(x)=\sum_{z\in\mathcal Z}
Q_x(1| z)\,Q_{x_0}(z)$.

The statistic $Z_t(x,\hat x(t))$ measures the minimal empirical
KL divergence to a model under which arm $x$
is at least as good as the current empirical best arm $\hat x(t)$.

{\bf Stopping condition.}
Let $\beta(t,\delta)$ be a nondecreasing threshold.
Define the stopping time
\begin{equation}
\label{eq:stopping_time}
\tau_\delta
:=
\inf
\{
t\ge 1:\;
\min\nolimits_{x\neq \hat x(t)}
Z_t(x,\hat x(t))
\;\ge\;
\beta(t,\delta)
\},
\end{equation}
and output $\hat x(\tau_\delta)$.
A concrete choice of $\beta(t,\delta)$ ensuring $\delta$-correctness
is provided in Appendix~\ref{app:stopping}.

\section{Theoretical Guarantees}
\label{sec:analysis}

In this section, we establish
$\delta$-correctness and asymptotic optimality of
Algorithm~\ref{alg:tas_ndpo}.
We first assume 
\begin{assumption}[Uniform interior (KL-regularity)]
\label{ass:uniform_interior}
There exists $\varepsilon\in(0,1/2)$ such that for every model $R\in\mathcal P$,
every arm $x\in\mathcal X$, and every mediator value $z\in\mathcal Z$,
\begin{equation}
R_x(z)\ge \varepsilon,
R_x(1| z)\in[\varepsilon,1-\varepsilon].
\end{equation}
\end{assumption}

This assumption guarantees that all KL divergences in the
characteristic-time optimization are finite and continuous.
It ensures that the relevant probability vectors remain in the interior
of the simplex, which is needed for the continuity and compactness
arguments in the asymptotic analysis.

Then, our algorithm has $\delta$-correctness.
\begin{restatable}{theorem}{Theoremthree}[$\delta$-correctness]
\label{thm:delta_correct}
Assume $\beta(t,\delta)$ is chosen as in Appendix~\ref{app:stopping}.
For any $\delta\in(0,1)$ and any $P\in\mathcal P$,
Algorithm~\ref{alg:tas_ndpo} satisfies
\begin{equation}
\mathbb{P}_P\!\left(
\tau_\delta < \infty, \hat x(\tau_\delta) \neq x^\star(P)
\right)
\le \delta.
\end{equation}
\end{restatable}

For the asymptotic optimality result, we also assume uniqueness of the
characteristic allocation.

\begin{assumption}[Unique characteristic allocation]
\label{ass:unique_char_alloc}
The maximizer $w^\star(P)$ of the characteristic optimization problem is unique.
\end{assumption}

We show the following lemmas:

\begin{restatable}{lemma}{Lemmaone}
\label{lem:all_cells_diverge}
Under Assumption~\ref{ass:uniform_interior}.
Let $h(t)=\lceil t^a\rceil$ for some $a\in(0,1)$ and let $g(t)$ be nondecreasing
with $g(t)\to\infty$ and $g(t)=o(h(t))$.
Under the sampling rule of Section~\ref{subsec:sampling_rule},
\begin{equation}
N_x(t)\to\infty, N_{x,z}(t)\to\infty,
\forall x\in\mathcal X,\ z\in\mathcal Z, \text{a.s.}
\end{equation}
\end{restatable}

Since $N_{x,z}(t)\to\infty$ for all $x,z$,
the strong law yields
\begin{equation}
\hat P_x(z;t)\to P_x(z),
\hat P_x(1| z;t)\to P_x(1| z)
\text{ a.s.}
\end{equation}
Hence $\hat P(t)\to P$ almost surely.

\begin{restatable}{lemma}{Lemmatwo}[Consistency of plug-in allocation]
\label{lem:plugin_consistency}
Under Assumption~\ref{ass:unique_opt}, if $\hat P(t)\to P$ almost surely,
then
$\hat w(t)\to w^\star(P)$ almost surely.
\end{restatable}

\begin{restatable}{lemma}{Lemmathree}[Sublinearity of forced rounds]
\label{lem:forced_sublinear}
Let $F(t)$ denote the number of rounds up to time $t$ in which either
competitor coverage or cell-level forcing overrides the D-tracking update.
Then
$F(t)=O(h(t)) + O(g(t))$.
In particular, if $h(t)=t^a$ with $a\in(0,1)$ and $g(t)=o(h(t))$,
then $F(t)=o(t)$ almost surely.
\end{restatable}

\begin{restatable}{lemma}{Lemmafour}[Tracking of sampling proportions]
\label{lem:tracking_final}
Under Assumptions~\ref{ass:uniform_interior} and~\ref{ass:unique_char_alloc},
the sampling proportions satisfy
$N_x(t)/t
\to w_x^\star(P)$  for all $x\in\mathcal X$ almost surely.
\end{restatable}
This follows from Lemmas~\ref{lem:plugin_consistency} and \ref{lem:forced_sublinear}.

Let $T^\star(P)$ denote the characteristic time in Theorem~\ref{thm:nde_lower_bound}.

\begin{restatable}{lemma}{Lemmafive}[Asymptotic growth of the GLRT] \label{lem:glrt_growth} 
For every competitor $x\neq x^\star(P)$, \begin{equation} 
\begin{aligned}
&\mathbb P_P\Big(
\liminf_{t\to\infty} t^{-1} Z_t(x,\hat x(t))
\;\ge\;\\
&\inf_{\substack{Q\in\mathcal P:
\theta_Q(x)\ge \theta_Q(x^\star(P))}}
\sum_{a\in\mathcal X} w_a^\star(P)\,
\mathrm{KL}(P_a\Vert Q_a)
\Big)=1.
\end{aligned} 
\end{equation} 

Consequently, 
\begin{equation}
\mathbb P_P\!\left(
\liminf_{t\to\infty}
t^{-1}
\min_{x\neq x^\star(P)} Z_t(x,\hat x(t))
\;\ge\;
{T^\star(P)}^{-1}
\right)=1.
\end{equation}
\end{restatable}

Then, our algorithm has asymptotic optimality.
\begin{restatable}{theorem}{Theoremfive}[Almost sure asymptotic optimality]
\label{thm:asympt_opt_as}
Fix a true model $P \in \mathcal P$ satisfying
Assumptions~\ref{ass:unique_opt}, ~\ref{ass:uniform_interior} and ~\ref{ass:unique_char_alloc}.
Then Algorithm~\ref{alg:tas_ndpo} satisfies
\begin{equation}
\mathbb P_P\!\left(
\limsup_{\delta\to0}
\frac{\tau_\delta}{\mathrm{kl}(\delta,1-\delta)}
\le
T^\star(P)
\right)=1.
\end{equation}
\end{restatable}


Finally, under Assumption \ref{idenNDE}, the expected NDPO is identified by $\theta$. Therefore, the $\delta$-correctness and asymptotic optimality of our algorithm for maximizing $\theta$ translate directly into $\delta$-correctness and asymptotic optimality for the expected NDPO under Assumption \ref{idenNDE}.

\section{Experiments}
\label{sec:experiments}

We evaluate our method on simulations based on synthetic and real-world data.
The real-world experiments use two datasets: the IPinYou advertising dataset
and the framing dataset from the \texttt{mediation} package. 
Additional
synthetic experiments are reported in Appendix~\ref{syntheticexperiments}.
These experiments empirically support the validity of the lower bound in
Theorem~\ref{thm:nde_lower_bound}, illustrate cases where the arm maximizing
$\mathbb{E}[Y|do(X=x)]$ differs from the expected NDPO-optimal arm, and study
sensitivity to gap structure, mediator sparsity, baseline-arm choice, and
mediator cardinality. Runtime comparisons are reported in
Appendix~\ref{app:runtime_comparison}.

\begin{table}[t]
\centering
\caption{
Stopping-time statistics on the IPinYou.
The error rate is the proportion of runs in which the algorithm chooses a suboptimal-arm.
Lower values show better performance.
}
\label{tab:ipinyou_results}
\scalebox{1}{
\begin{tabular}{l|ccc}
\toprule
Method
& Median
& Interquartile range
& Error \\
\hline\hline
\makecell[l]{\textbf{TaS-NDPO} \\ (\textbf{ours})}
& \textbf{12{,}928}
& \textbf{[4{,}864,\;33{,}984]}
& \textbf{0.00} \\

\makecell[l]{TaS-NDPO \\ (arm-level)}
& 25{,}856
& [11{,}520,\;38{,}848]
& 0.00 \\

Baseline-First
& 18{,}304
& [8{,}960,\;38{,}016]
& 0.01 \\

Uniform
& 21{,}248
& [9{,}472,\;38{,}016]
& 0.00 \\
\bottomrule
\end{tabular}
}
\end{table}


{\bf Dataset.}
We use the IPinYou dataset \citep{zhang2015realtimebiddingbenchmarkingipinyou}, a publicly available benchmark for computational advertising, available at (\url{https://contest.ipinyou.com}).
The dataset contains approximately $N=3.16\times 10^6$ ad impressions.
From this dataset, we select $K=10$ creatives with sufficient support (i.e., each with at least 30,000 impressions).
The mediator $Z$ corresponds to discretized slot visibility with three levels,
and the outcome $Y$ is a binary click indicator.
We impose Assumption \ref{idenNDE}, corresponding to the absence of unmeasured mediator–outcome confounding under treatment interventions, which guarantees identification of the expected NDPO from interventional data.

{\bf Algorithms.}
We compare the following algorithms:
\begin{itemize}
    \item \textbf{TaS-NDPO (ours).} Algorithm \ref{alg:tas_ndpo}.

    \item \textbf{TaS-NDPO (arm-level).}
    Variant of TaS-NDPO by replacing cell-level forcing with classical arm-level forced exploration.

    \item \textbf{Baseline-First.}
    The algorithm first estimates the baseline distribution
    $P_{x_0}(z)$ using uniform sampling, and then samples arms uniformly to estimate 
    $P_x(1| z)$.
    \item \textbf{Uniform.}
    Uniform sampling over arms combined with the same plug-in estimator and stopping rule as ours.
\end{itemize}
To avoid zero empirical probabilities when some counts are small,
we use Laplace smoothing \citep{Manning2008}, i.e.,
$\hat P_x(z;t)=(N_{x,z}(t)+\alpha)/(N_x(t)+\alpha|\mathcal Z|)$
with $\alpha=0.05$.



{\bf Results.}
We first compare the arms that maximize the interventional mean $\mathbb{E}[Y|do(X=x)]$ with those that maximize the expected NDPO.
On the IPinYou dataset, 
creative~10,722 (ID) achieves the highest interventional mean, while creative~10,720 maximizes the expected NDPO.
Examining mediator distributions reveals that creative~10,722 appears in the top slot $23.7\%$ of the time, compared to only $9.9\%$ for creative~10,720.
Thus, part of its click-rate advantage is mediated through favorable slot positioning, which is removed under the NDPO perspective.

We next evaluate all BAI algorithms with a confidence level $\delta=0.05$ over 100 independent runs.
Table~\ref{tab:ipinyou_results} summarizes stopping time statistics, and Figure~\ref{fig:ipinyou_ecdf} reports the empirical cumulative distribution function (ECDF) of stopping times.
As shown in Table~\ref{tab:ipinyou_results}, all methods exhibit empirical error rates below 0.05. Thus, all methods empirically satisfy the $\delta$-correctness requirement.
The TaS-NDPO (ours) achieves a median stopping time of 12{,}928 samples, reducing sample complexity by roughly 50\% compared to the arm-level variant (25{,}856).
It also improves upon uniform sampling and the two-stage baseline while maintaining zero empirical error.
Figure~\ref{fig:ipinyou_ecdf} further illustrates the distributional behavior of the stopping times: the ECDF curve of TaS-NDPO lies uniformly above the alternatives, indicating consistently faster identification across runs.
The separation is most pronounced in the lower quantiles, demonstrating that cell-level exploration in our algorithm accelerates early evidence accumulation. 
These results support that our algorithm substantially improves sample efficiency for identifying the optimal-arm w.r.t. the expected NDPO.

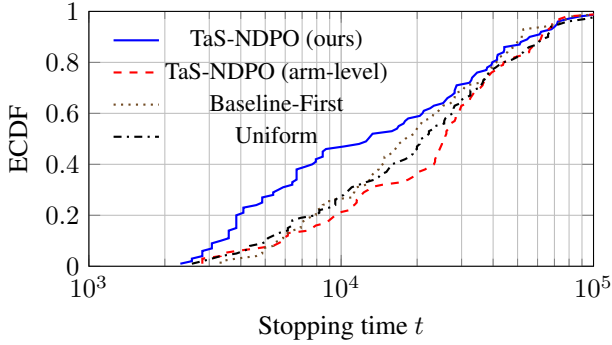
\begin{figure}[t]
\centering
\begin{tikzpicture}
\begin{axis}[
    width=1\linewidth,
    height=0.6\linewidth,
    xmode=log,
    xmin=1e3, xmax=1e5,
    ymin=0, ymax=1,
    xlabel={Stopping time $t$},
    ylabel={ECDF},
    legend style={
    font=\footnotesize,
    at={(0.03,0.97)},
    anchor=north west,
    draw=none,
    fill=none,
    row sep=0.5pt
},
    grid=both,
]

\addplot+[thick, mark=none]
table[x=x, y=y, col sep=comma]
{ecdf_ours_cell.csv};
\addlegendentry{TaS-NDPO (ours)}

\addplot+[thick, dashed, mark=none]
table[x=x, y=y, col sep=comma]
{ecdf_arm_forced.csv};
\addlegendentry{TaS-NDPO (arm-level)}

\addplot+[thick, dotted, mark=none]
table[x=x, y=y, col sep=comma]
{ecdf_baseline_first.csv};
\addlegendentry{Baseline-First}

\addplot+[thick, dashdotted, mark=none]
table[x=x, y=y, col sep=comma]
{ecdf_uniform_arm.csv};
\addlegendentry{Uniform}

\end{axis}
\end{tikzpicture}
\caption{
ECDF of stopping times on the IPinYou.
Higher curves correspond to faster identification.
}
\label{fig:ipinyou_ecdf}
\end{figure}

{\bf Additional real-world validation.}
We also evaluate the algorithms on the framing dataset from the
\texttt{mediation} package at (\url{https://cran.r-project.org/web/packages/mediation/index.html}), using the same evaluation protocol as above.
Table~\ref{tab:framing_results} reports the median stopping time,
interquartile range, and empirical error rate. TaS-NDPO achieves the lowest
median stopping time and zero empirical error, providing additional
real-world validation beyond the IPinYou dataset.

\begin{table}[t]
\centering
\caption{
Stopping-time statistics on the framing dataset.
The error rate is the proportion of runs in which the algorithm chooses a
suboptimal arm. Lower values show better performance.
}
\label{tab:framing_results}
\begin{tabular}{l|ccc}
\toprule
Method
& Median
& Interquartile range
& Error \\
\hline\hline
\makecell[l]{\textbf{TaS-NDPO} \\ (\textbf{ours})}
& \textbf{5{,}888}
& \textbf{[2{,}624,\;9{,}408]}
& \textbf{0.00} \\

\makecell[l]{TaS-NDPO \\ (arm-level)}
& 6{,}400
& [3{,}840,\;10{,}304]
& 0.00 \\

Baseline-First
& 8{,}704
& [7{,}168,\;9{,}984]
& 0.00 \\

Uniform
& 7{,}680
& [3{,}392,\;11{,}518]
& 0.02 \\
\bottomrule
\end{tabular}
\end{table}

The additional sensitivity experiments in
Appendix~\ref{syntheticexperiments} show that the benefit of cell-level
forcing is regime-dependent. The gains are most substantial in settings
with rare mediator--outcome cells, while the arm-level Track-and-Stop
variant remains competitive in easier regimes. The gap-scaling experiment
varies the NDPO gap while holding the mediator structure fixed.

\section{Conclusion}

We study the fixed-confidence BAI problem of selecting the arm that maximizes the expected NDPO and develop a TaS–based algorithm that is $\delta$-correct and asymptotically optimal.
We then evaluate its efficiency through simulation experiments.

We focus on the expected NDPO in this paper.
Other related quantities include the controlled direct effect and the natural indirect effect, defined as $\mathbb{E}[Y_{x,z}] - \mathbb{E}[Y_{x_0,z}]$ and $\mathbb{E}[Y_{x_0,Z_x}] - \mathbb{E}[Y_{x_0}]$, respectively \citep{Pearl2001}. 
In Appendix \ref{sec_niecde}, we consider the corresponding BAI problems for optimizing $\mathbb{E}[Y_{x,z}]$ over $x$ (for fixed $z$) and for optimizing $\mathbb{E}[Y_{x_0,Z_x}]$ over $x$. 
The former reduces to a standard BAI problem, whereas the latter closely parallels our algorithm for the expected NDPO.
Furthermore, we focused on the case of a binary outcome since it covers a wide range of applications related to website optimization. Extending our statistical framework to multi-category outcomes (e.g., five-star reviews) is relatively easy. 

An interesting direction for future work is to extend these ideas to path-specific objectives in more general causal structures \citep{Avin2005,Shpitser2018,Malinsky2019}.

\begin{acknowledgements} 
The authors thank the anonymous reviewers for their time and thoughtful comments.
J. Komiyama was supported by the MBZUAI Start-up Fund [BF0121].
\end{acknowledgements}

\bibliography{uai2026-template}

\newpage
\onecolumn
\appendix

\title{Fixed-Confidence Best-Arm Identification for Causal Mediation Analysis\\(Supplementary Material)}
\maketitle

\section{Motivating Examples}
\label{app-moti}

We present motivating examples illustrating why we study the maximization of the expectation of NDPO.

{\bf Example 1 (Advertisement).} 
We present a motivating example illustrating why we study the maximization of the expectation of NDPO.
We consider the problem of selecting an advertisement creative ($X$) from multiple candidates, where the outcome ($Y$) represents a click indicator, purchase amount, or conversion.
Researchers often select the optimal advertisement by maximizing the expected potential outcome under arm $x$.
In practice, however, an advertisement may improve the outcome $Y$ partly through undesirable mechanisms-for example by exploiting user confusion, eliciting accidental clicks, or leveraging platform-specific presentation artifacts.
We treat such undesired factors as a mediator ($Z$).
Improving outcomes primarily by changing $Z$ should not be considered a legitimate success (e.g., due to ethical or regulatory concerns).
In reality, for example, \citet{10.1145/3487552.3487850} pointed out the widespread use of misleading and manipulative tactics in online political advertising.
The NDPO ignores effects achieved through such undesirable factors. 
Consequently, selecting the optimal advertisement by maximizing the NDPO leads to fairer advertising decision-making, based purely on the direct effect rather than on effects mediated by undesirable factors.



{\bf Example 2 (Medicine).} 
In medical decision-making, physicians aim to select treatments that optimally improve patient outcomes. However, some treatments may achieve their effects through adverse intermediate responses (e.g., fever or other adverse reactions). 
In such cases, it is desirable to evaluate treatment effects that exclude pathways operating through these adverse mediators. 
Researchers therefore seek to maximize the direct effect, thereby maximizing the treatment benefit while avoiding pathways through these adverse mediators. 
The indirect effect is also of interest, as it quantifies the effect transmitted through these adverse mediators, and minimizing this effect is another important research objective. 
In this paper, however, we focus on maximizing the direct effect.

{\bf Example 3 (Politics).} 
In political decision-making, policymakers often aim to select policies that effectively reduce inflation. 
However, some policies may achieve this through undesirable mechanisms, such as increasing unemployment or reducing economic growth. In such cases, it may be preferable to evaluate policies based on their effects that exclude these undesirable pathways.
Although the direct effect cannot generally be reproduced by real-world interventions, it provides a useful measure of the desirability of treatment policies under a preference for avoiding specific mediating pathways.
In some applications, such as political science, researchers may focus on the indirect effect to understand the underlying mechanisms through which an intervention exerts its influence. 
In contrast, our objective is to identify the treatment policy that is most preferred under this criterion.

{\bf Example 4 (Education).} In educational decision-making, teachers often choose among candidate classes or programs to improve students’ performance. However, some classes may increase test scores through undesirable mechanisms, such as excessively stressful learning environments. In such cases, it may be preferable to evaluate options based on their effects that exclude these undesirable pathways.

{\bf Example 5 (Legal).} In legal decision-making, policymakers often design laws or regulations to reduce crime. However, some policies may achieve this through undesirable mechanisms, such as excessive surveillance. In such cases, it may be preferable to evaluate policies based on their effects that exclude these undesirable pathways.

\section{Stopping Condition}
\label{app:stopping}

This appendix justifies the stopping rule introduced in
Section~\ref{sec:algorithm}, and in particular the choice of the confidence
radius $\beta(t,\delta)$ ensuring $\delta$-correctness.

\subsection{Uniform concentration for discrete distributions}

We begin with a standard concentration inequality for empirical distributions
on a finite alphabet.

\begin{lemma}[Concentration on a finite alphabet]
\label{lem:alphabet_conc}
Let $\hat P_n$ be the empirical distribution obtained from $n$ i.i.d.\ samples
drawn from a categorical distribution $P$ supported on an alphabet of size $M$.
Then, for any $u>0$,
\begin{equation}
\Pr\!\left(\mathrm{KL}(\hat P_n \Vert P) \ge u\right)
\;\le\;
(n+1)^M e^{-n u}.
\end{equation}
\end{lemma}

Lemma~\ref{lem:alphabet_conc} is a classical concentration inequality for discrete symbols (see, e.g., \citealp{Cover2009}).

\begin{lemma}[Concentration on a combination of finite alphabets]
\label{lem:alphabet_conc_comb}
Let $i = 1,2,\dots,I$ for some finite positive integer $I$.
For each $i$, let $\hat P_{n_i,i}$ be the empirical distribution obtained from $n_i$ i.i.d.\ samples
drawn from a categorical distribution $P_i$ supported on an alphabet of size $M_i$.
Let $M = \max_i M_i$.
Then, for any $u>0$ and for any $(n_1, n_2, \dots, n_I)$ such that $n = \sum_{i=1}^I n_i$, 
\begin{equation}
\Pr\!
\left(
\sum_{i = 1}^I n_i \mathrm{KL}(\hat P_{n_i,i} \Vert P_i) \ge n u
\right)
\;\le\;
(n+1)^{3MI} e^{-n u}
\label{ineq_concentration_comb}
\end{equation}
holds.
\end{lemma}
\begin{proof}
Let $\mathcal{V}$ be the number of possible combinations of empirical means $(\hat{P}_{n_1,1}, \hat{P}_{n_2,2}, \dots, \hat{P}_{n_I,I})$ such that $\sum_i n_i = n$. Since there are $I$ empirical distribution of categories at most $M$ and each category's empirical mean is in $\{0, 1/n_i, 2/n_i, \dots, 1\}$ such that $n_i \le n$, we have $|\mathcal{V}| \le n^{2MI}$.
Let 
\begin{equation}
\mathcal{P}_u
=
\left\{
(\hat{P}_{n_1,1}, \hat{P}_{n_2,2}, \dots, \hat{P}_{n_I,I}) \in \mathcal{V}: 
\sum_{i = 1}^I n_i \mathrm{KL}(\hat P_{n_i,i} \Vert P_i) \ge n u
\right\}.
\end{equation}
Then, Lemma \ref{lem:alphabet_conc} states that the probability that each such set of empirical distributions in $\mathcal{P}_u$ realizes is at most 
\begin{equation}\label{ineq_singlecomb}
\prod_i (n_i+1)^{M_i} e^{-n_i u} \le (n+1)^{MI} e^{-n u},
\end{equation}
and thus
\begin{align}
\Pr\!
\left(
\sum_{i = 1}^I n_i \mathrm{KL}(\hat P_{n_i,i} \Vert P_i) \ge n u
\right)
&\le
|\mathcal{P}_u| \times \text{(The right-hand side of \eqref{ineq_singlecomb})}\\
&\le n^{2MI} \times (n+1)^{MI} e^{-nu}\\
&\le (n+1)^{3MI} e^{-nu}.
\end{align}
\end{proof}

\subsection{Application to mediator and outcome distributions}

We apply Lemma~\ref{lem:alphabet_conc_comb} to both components of the causal model.
For each arm $x$, let $\hat P_{Z| x,n}$ denote the empirical distribution of
the mediator $Z$ after $n$ pulls of arm $x$.
Let $M_Z := |\mathcal Z|$ and $M_Y := |\mathcal Y|$ denote the alphabet sizes of
the mediator and outcome, respectively. In our setting $M_Y=2 < M_Z$.

\paragraph{Time-uniform confidence allocation.}
To obtain bounds that hold uniformly over all $t\ge 1$,
define
\begin{equation}
\delta_t := \frac{\delta}{\pi^2 t^{2 + 2K M_Z}}.
\end{equation}
Since $\sum_{t=1}^\infty 1/(\pi^2 t^2) \le 1$,
a union bound over all $t$ and all possible combinations of $(N_x(t), N_{x,z}(t))$ (number of such possible combinations is at most $t^{2K M_Z}$) yields an event of probability at least $1-\delta$.
Namely, by setting
\begin{align}
\label{eq:beta}
\beta(t,\delta)
&=
\log\Big(\tfrac{\pi^2 t^{2+2K M_Z}}{\delta}\Big)
+
6 M_Z K \log(t+1),
\end{align}
we have
\begin{align}
\forall t
\{
\mathcal L_t(Q) \le \beta(t,\delta)
\}
\end{align}
holds with probability at least $1-\delta$.

\section{IMPLEMENTATION DETAILS FOR THE REDUCED INNER OPTIMIZATION}
\label{app:inner_qp}

This appendix describes the numerical solution of the reduced inner
optimization problem appearing in the plug-in allocation
\eqref{eq:plugin_allocation} and in the GLRT statistic
\eqref{eq:glrt_stat}, after reduction to a single competitor and
restriction to the relevant arms.

\subsection{Reduced weighted inner problem}

Fix a competitor $x'\neq x^\star$ and let
\begin{equation}
\gamma_0 := w_{x_0},
\qquad
\alpha_{s,z} := w_{x^\star}\,\hat P_{x^\star}(z),
\qquad
\alpha_{p,z} := w_{x'}\,\hat P_{x'}(z),
\end{equation}
where $w$ denotes the current allocation vector.

Let
\begin{equation}
P_0(z) := \hat P_{x_0}(z), \quad
p_s(z) := \hat P_{x^\star}(1| z), \quad
p_p(z) := \hat P_{x'}(1| z).
\end{equation}

After the three-arm reduction (Section~\ref{subsec:optimization_details}),
the inner problem reduces to
\begin{align}
\min_{\substack{q_0 \in \Delta(\mathcal Z)\\
q_s, q_p \in (0,1)^{|\mathcal Z|}}}
\quad &
\gamma_0\,\mathrm{KL}(P_0 \Vert q_0)
+
\sum_{z\in\mathcal Z}
\alpha_{s,z}\,
\mathrm{KL}(p_s(z)\Vert q_s(z))
\nonumber\\
&\quad+
\sum_{z\in\mathcal Z}
\alpha_{p,z}\,
\mathrm{KL}(p_p(z)\Vert q_p(z))
\label{eq:reduced_inner_weighted}
\\
\text{s.t.}\quad &
\sum_{z\in\mathcal Z}
q_0(z)\,(q_p(z)-q_s(z))
\;\ge\;
0.
\label{eq:reduced_inner_constraint}
\end{align}

The feasible set is compact and the objective is continuous.
Moreover, the problem is convex in $(q_s,q_p)$ for fixed $q_0$
and convex in $q_0$ for fixed $(q_s,q_p)$,
but not jointly convex.

We solve \eqref{eq:reduced_inner_weighted}--\eqref{eq:reduced_inner_constraint}
via block coordinate descent (alternating optimization).

\subsection{Update of $(q_s,q_p)$ for fixed $q_0$}

For fixed $q_0$, the subproblem in $(q_s,q_p)$ is convex:
\begin{align}
\min_{q_s,q_p}
\quad &
\sum_{z}
\alpha_{s,z}\,\mathrm{KL}(p_s(z)\Vert q_s(z))
+
\sum_{z}
\alpha_{p,z}\,\mathrm{KL}(p_p(z)\Vert q_p(z))
\label{eq:qs_qp_subproblem}
\\
\text{s.t.}\quad &
\sum_{z} q_0(z)(q_p(z)-q_s(z)) \ge 0.
\nonumber
\end{align}

If the constraint is inactive at the unconstrained minimizer
$(q_s,q_p)=(p_s,p_p)$, then this is the solution.

Otherwise, the constraint is active and we introduce a
Lagrange multiplier $\lambda\ge 0$.
For each $z$, the subproblem separates and we must solve
\begin{equation}
\min_{q\in(0,1)}
\alpha\,\mathrm{KL}(p\Vert q) + b\,q,
\end{equation}
where $b=\lambda q_0(z)$ (with opposite signs for $q_s$ and $q_p$).

Using
\begin{equation}
\frac{\partial}{\partial q}\mathrm{KL}(p\Vert q)
=
\frac{q-p}{q(1-q)},
\end{equation}
the stationarity condition yields the quadratic equation
\begin{equation}
b q^2 - (\alpha + b) q + \alpha p = 0.
\end{equation}

Among its real roots in $(0,1)$, the one minimizing the
objective is selected.
The multiplier $\lambda$ is chosen by one-dimensional
bisection so that the constraint
\eqref{eq:reduced_inner_constraint}
holds with equality when active.
Monotonicity of the constraint function in $\lambda$
ensures existence and uniqueness of the solution.

\subsection{Update of $q_0$ for fixed $(q_s,q_p)$}

For fixed $(q_s,q_p)$, define
\begin{equation}
d(z) := q_p(z) - q_s(z).
\end{equation}
The subproblem in $q_0$ is
\begin{align}
\min_{q_0 \in \Delta(\mathcal Z)}
\quad &
\mathrm{KL}(P_0 \Vert q_0)
\text{ s.t. }
\sum_{z} q_0(z)\, d(z) \ge 0.
\label{eq:q0_subproblem}
\end{align}

If the constraint is satisfied at $q_0=P_0$, then this is optimal.

Otherwise, introducing a multiplier $\mu\ge 0$
and using KKT conditions yields the closed-form solution
\begin{equation}
q_0(z)
=
\frac{P_0(z)}{\nu - \mu d(z)},
\end{equation}
where $\nu$ is determined by the normalization condition
$\sum_z q_0(z)=1$.
The multiplier $\mu$ is obtained by one-dimensional
bisection so that the constraint holds with equality
when active.

\subsection{Alternating optimization and convergence}

The algorithm alternates between:

\begin{itemize}
\item updating $(q_s,q_p)$ given $q_0$,
\item updating $q_0$ given $(q_s,q_p)$.
\end{itemize}

Each block subproblem is solved exactly and is convex.
The objective in \eqref{eq:reduced_inner_weighted}
is non-increasing at every iteration,
and the feasible set is compact.
Therefore, by standard results on block coordinate descent
for bi-convex problems \citep{Tseng2001},
every limit point of the sequence is a stationary point
of \eqref{eq:reduced_inner_weighted}--\eqref{eq:reduced_inner_constraint}.

In practice, the dimensionality is $O(|\mathcal Z|)$
and convergence is rapid.
This procedure is used both in the plug-in allocation
computation and in evaluation of the GLRT statistic.

\section{BAI Algorithms for other mediation measures}
\label{sec_niecde}

Finding the best arm in terms of $\mathbb{E}[Y_{x,z}]$, which is the first term of the controlled direct effect (CDE) at $Z=z$, 
is relatively straightforward using a best arm identification algorithm.
For ease of discussion, let us consider $\mathbb{E}[Y_{x,z}]$ at $z=1$
\begin{equation}
\mathbb{E}_P\!\left[Y_{x,z=1}\right].
\end{equation}
This boils down to standard bandits/BAI but we only count the sample when it is shown at a specific position $z = 1$. Namely,
\begin{align}
N_x(t) = \sum_{s=1}^t \mathbf{1}\{X_t = x, Z = 1\}, 
\hat{P}_x(t) = 
\frac{\sum_{s=1}^t Y_t \mathbf{1}\{X_t = x, Z = 1\}}{N_x(t)}
\end{align}
and one can run Track-and-Stop or Top-two Thompson sampling \citep{russottts} based on these statistics.

Regarding the Natural Indirect effect (NIE),
the first term of NDE is
\begin{align}
\mathbb{E}_P\!\left[Y_{x_0,Z_x}\right]= \sum_z P_{x_0}(1|z) P_x(z).
\end{align}

We can consider the alternative distributions as:
\begin{align}
\mathrm{Alt}(P)
&:=
\left\{
Q \in \mathcal{P} :
\exists x \neq x^*(P) 
\text{ such that } 
\sum_{z \in \mathcal{Z}} Q_{x_0}(1|z) Q_x(z) 
>
\sum_{z \in \mathcal{Z}} Q_{x_0}(1|z) Q_{x^*(P)}(z) 
\right\}.
\label{eq:altP_nie}
\end{align}
We hypothesize that our Track-and-Stop algorithm with a cutting-plane procedure (Section~\ref{subsec:optimization_details}) can also achieve asymptotically optimal sample complexity for identifying the best arm w.r.t. the $\mathbb{E}[Y_{x_0,Z_x}]$-criterion.
In summary, finding the best arm for $\mathbb{E}[Y_{x,z}]$ is straightforward. Finding the best arm for $\mathbb{E}[Y_{x_0,Z_x}]$ can be done using our framework.


\section{Synthetic Experiments}
\label{syntheticexperiments}


In this appendix, we conduct synthetic experiments.

\paragraph{Setting.}
We consider instances with $K=3$ arms and $M=3$ mediator values.
Full parameter specifications are as follows:
Let $x_0=0$ and $\mathcal Z=\{0,1,2\}$. Fix
$P_0(Z)=(0.8,0.15,0.05)$.
Let $P(Z| X=1)=(0.1,0.2,0.7)$ and $P(Z| X=2)=P(Z| X=0)=P_0(Z)$.
Set $P(Y=1| X=0,Z=z)=0.2$,
$P(Y=1| X=1,Z\in\{0,1\})=0.2$, $P(Y=1| X=1,Z=2)=0.9$,
and $P(Y=1| X=2,Z=z)=0.235+\Delta$ for all $z$.
Then $\theta(1)=0.235$ and $\theta(2)=0.235+\Delta$, so the expected NDPO gap equals $\Delta$,
while $\mathbb E[Y_1]=0.69$ is the largest interventional mean.
To avoid zero empirical probabilities when some counts are small,
we use Laplace smoothing \citep{Manning2008}, i.e.,
$\hat P_x(z;t)=(N_{x,z}(t)+\alpha)/(N_x(t)+\alpha|\mathcal Z|)$
with $\alpha=0.05$.
All methods use the same confidence level $\delta=0.05$
and the same stopping threshold $\beta(t,\delta)$
from Appendix~\ref{app:stopping}.
Optimization tolerances and forcing schedules are matched whenever applicable.
Performance was robust to moderate changes in these settings.

\paragraph{Algorithms.}
We compare the following algorithms:

\begin{itemize}
    \item \textbf{TaS-NDPO (ours).} Algorithm \ref{alg:tas_ndpo}.

    \item \textbf{TaS-NDPO (arm-level).}
    Variant of TaS-NDPO by replacing cell-level forcing with classical arm-level forced exploration.

    \item \textbf{Baseline-First.}
    The algorithm first estimates the baseline distribution
    $P_{x_0}(z)$ using uniform sampling, and then samples arms uniformly to estimate 
    $P_x(1| z)$.

    \item \textbf{Uniform.}
    Uniform sampling over arms combined with the same plug-in estimator and stopping rule as ours.
    \item 
\textbf{TaS-IM (diagnostic).}
TaS algorithm for optimizing the interventional mean (IM) $\mathbb{E}[Y_x]$.
\end{itemize}

\paragraph{Gap-scaling family.}
Mediator distributions are fixed while the NDPO gap
\begin{equation}
\Delta := \theta(x^\star)-\max_{x\neq x^\star}\theta(x)
\end{equation}
is controlled through the outcome model of arm~2.
This isolates the effect of gap magnitude.

To validate Theorem~\ref{thm:nde_lower_bound}, we vary $\Delta$
and measure the median stopping time.
Figure~\ref{fig:gap_scaling} shows log--log scaling.
The empirical slope aligns with the reference $1/\Delta^2$,
confirming the predicted instance-dependent complexity.

\begin{figure}[t]
\centering
\includegraphics[width=0.5\linewidth]{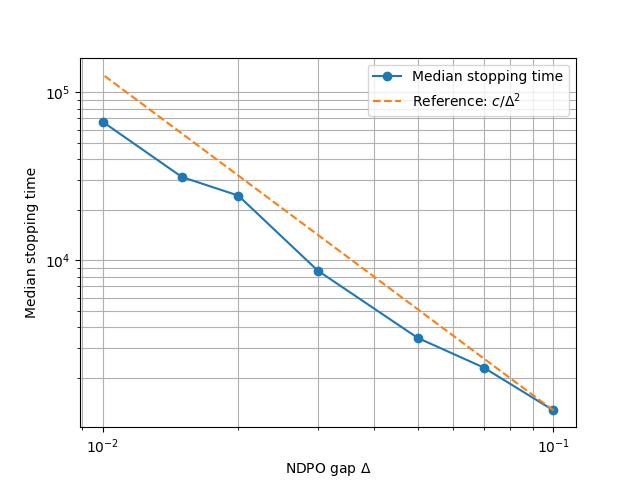}
\caption{Median stopping time vs.\ NDPO gap $\Delta$ (log--log scale).
Dashed line: $1/\Delta^2$.}
\label{fig:gap_scaling}
\end{figure}

\paragraph{Decision differences.}
In this model, the interventional-mean-optimal arm differs from the expected NDPO-optimal arm.
TaS-IM selects arm~1, while TaS-NDPO consistently identifies arm~2,
demonstrating that optimizing expected NDPO can lead to different decisions.

\subsection{Additional sensitivity experiments}
\label{app:sensitivity_experiments}

We further evaluate the effect of mediator sparsity, mediator cardinality,
and baseline-arm choice. The gap-scaling experiment above varies the NDPO
gap while holding the mediator structure fixed. In contrast, the experiments
below vary the mediator structure or the reference arm. Overall, the results
show that the advantage of cell-level forcing is regime-dependent, with
substantial gains in settings with rare mediator--outcome cells. The
arm-level Track-and-Stop baseline already represents a strong adaptive
method and performs well in easier regimes.

\paragraph{Mediator sparsity.}
Table~\ref{tab:sparsity} reports median stopping times as the mediator
sparsity parameter $\eta$ varies. Smaller values of $\eta$ correspond to
rarer mediator--outcome cells. Our method is best or tied for best across
all values of $\eta$. The improvement over arm-level forcing becomes more
pronounced as rare cells become important, showing the benefit of explicit
cell-level forcing in sparse mediator regimes.

\begin{table}[t]
\centering
\caption{
Median stopping time as a function of mediator sparsity $\eta$.
Lower is better. Bold indicates the best method for each $\eta$.
}
\label{tab:sparsity}
\begin{tabular}{c|cccc}
\toprule
$\eta$
& TaS-NDPO (ours)
& TaS-NDPO (arm-level)
& Uniform
& Baseline-First \\
\midrule
0.5   & \textbf{2048}  & \textbf{2048} & 2560  & 6912  \\
0.1   & \textbf{20368} & 20608         & 28928 & 21664 \\
0.05  & \textbf{29184} & 36864         & 33664 & 31960 \\
0.01  & \textbf{36480} & 47104         & 46080 & 41216 \\
0.005 & \textbf{41216} & 57088         & 46208 & 47520 \\
0.001 & \textbf{45056} & 54528         & 47616 & 49024 \\
\bottomrule
\end{tabular}
\end{table}

\paragraph{Mediator cardinality.}
Table~\ref{tab:cardinality} reports median stopping times as the number of
mediator values $M$ varies. The results are not monotone in $M$, confirming
that the benefit of cell-level forcing depends on the structure of the
instance. Our method performs best for $M=2$ and $M=16$, while the
arm-level variant is best for $M=8$ and Baseline-First is slightly better
for $M=4$. This shows that arm-level exploration can be competitive in
easier regimes, whereas cell-level forcing is most useful when the relevant
mediator--outcome cells drive the difficulty of identification.

\begin{table}[t]
\centering
\caption{
Median stopping time as a function of mediator cardinality $M$.
Lower is better. Bold indicates the best method for each $M$.
}
\label{tab:cardinality}
\begin{tabular}{c|cccc}
\toprule
$M$
& TaS-NDPO (ours)
& TaS-NDPO (arm-level)
& Uniform
& Baseline-First \\
\midrule
2  & \textbf{13312} & 15360          & 15872 & 13440 \\
4  & 21632          & 23592          & 25432 & \textbf{21504} \\
8  & 21504          & \textbf{14976} & 18560 & 19584 \\
16 & \textbf{10752} & 20864          & 22016 & 22016 \\
\bottomrule
\end{tabular}
\end{table}

\paragraph{Baseline-arm choice.}
Table~\ref{tab:baseline_choices} reports median stopping times for different
choices of the baseline arm $x_0$. Changing $x_0$ changes the expected NDPO
target and may also change the NDPO-best arm. For $x_0=0$, the instance is
harder and our method achieves the smallest stopping time. For $x_0=1$ and
$x_0=2$, the instance is easier, and TaS-NDPO, the arm-level variant, and
Uniform all stop at the same median time. Baseline-First is slower in these
easy cases because it separately estimates the baseline distribution before
sampling the remaining arms.

\begin{table}[t]
\centering
\caption{
Median stopping time as a function of baseline-arm choice $x_0$.
Lower is better. Bold indicates the best method for each baseline choice
in a 4-arm instance.
}
\label{tab:baseline_choices}
\begin{tabular}{c|ccccc}
\toprule
$x_0$
& NDPO best arm
& TaS-NDPO (ours)
& TaS-NDPO (arm-level)
& Uniform
& Baseline-First \\
\midrule
0 & 3 & \textbf{59904} & 69248          & 70144          & 61952 \\
1 & 1 & \textbf{512}   & \textbf{512}   & \textbf{512}   & 5376  \\
2 & 2 & \textbf{512}   & \textbf{512}   & \textbf{512}   & 5376  \\
\bottomrule
\end{tabular}
\end{table}

\subsection{Runtime comparison}
\label{app:runtime_comparison}

We also measure the wall-clock runtime of each method. Table~\ref{tab:runtime}
compares statistical efficiency, measured by median stopping time, with
computational cost, measured by average runtime per run. TaS-NDPO incurs a
moderate computational overhead relative to non-adaptive baselines, but its
runtime is comparable to the arm-level adaptive baseline and it achieves the
lowest median stopping time.

\begin{table}[t]
\centering
\caption{
Comparison of statistical efficiency and computational cost.
The median stopping time measures sample complexity, and the average runtime
is measured in seconds per run. Lower is better.
}
\label{tab:runtime}
\begin{tabular}{lcc}
\toprule
Method
& Median stopping time
& Average runtime (s) \\
\midrule
TaS-NDPO (ours)
& \textbf{36{,}480}
& 22.5 \\

TaS-NDPO (arm-level)
& 47{,}104
& 25.3 \\

Uniform
& 46{,}080
& 6.75 \\

Baseline-First
& 46{,}336
& 6.43 \\
\bottomrule
\end{tabular}
\end{table}

\section{Proofs}
\label{appa}

In this appendix, we provide the proofs of the theorems, propositions, and lemmas stated in the main text.

\subsection{Proof of Theorem~\ref{thm:nde_identification}}
\label{app:NDPO_identifiability_proof}

\Theoremone*

\begin{proof}
Under Assumption \ref{idenNDE}, we have 
\begin{equation}
\begin{aligned}
&\mathbb{E}[Y_{x,Z_{x_0}}]\\
&=\sum_{y,z}y\mathbb{P}(Y_{x,z}=y|Z_{x_0}=z)\mathbb{P}(Z_{x_0}=z)\\
&=\sum_{y,z}y\mathbb{P}(Y_{x,z}=y)\mathbb{P}(Z_{x_0}=z)(\because\text{Assumption \ref{idenNDE}})\\
&=\sum_{y,z}y\mathbb{P}(Y_{x,z}=y|Z_x=z)\mathbb{P}(Z_{x_0}=z)(\because\text{Assumption \ref{idenNDE}})\\
&=\sum_{y,z}y\mathbb{P}(Y_{x}=y|Z_x=z)\mathbb{P}(Z_{x_0}=z)(\because\text{Counterfactual Consistency})\\
&=\sum_{y,z}y\mathbb{P}(Y=y|Z=z,do(X=x))\mathbb{P}(Z=z|do(X=x_0)).
\end{aligned}
\end{equation}
Then, we have the theorem.
\end{proof}

\subsection{Proof of Theorem~\ref{thm:nde_lower_bound}}
\label{app:lower_bound_proof}

\Theoremtwo*
\begin{proof}
The proof follows the standard change-of-measure argument \citep{lairobbins1985}.

Let $N_x(t)$ denote the (random) number of times arm $x$ is sampled up to time
$t$, and let $\tau_\delta$ be the stopping time of a $\delta$-correct algorithm.
For any alternative model $Q \in \mathrm{Alt}(P)$, the transportation inequality
(Lemma~1 in \citealp{garivier2016optimalbestarmidentification}) yields
\begin{equation}
\sum_{x=0}^{K-1}
\mathbb{E}_P\!\left[N_x(\tau_\delta)\right]
\mathrm{KL}(P_x \Vert Q_x)
\;\ge\;
\mathrm{kl}(\delta,1-\delta),
\end{equation}
where $P_x$ and $Q_x$ denote the joint distributions of $(Z,Y)$ under arm $x$
for models $P$ and $Q$, respectively.

Define the sampling proportions
\begin{equation}
w_x
:=
\frac{\mathbb{E}_P[N_x(\tau_\delta)]}{\mathbb{E}_P[\tau_\delta]},
\qquad
\sum_x w_x = 1.
\end{equation}
Dividing both sides by $\mathbb{E}_P[\tau_\delta]$ gives
\begin{equation}
\mathbb{E}_P[\tau_\delta]
\sum_x w_x \mathrm{KL}(P_x \Vert Q_x)
\;\ge\;
\mathrm{kl}(\delta,1-\delta).
\end{equation}

Under the joint mediator--outcome observation model, the KL divergence for a
fixed arm $x$ decomposes as
\begin{equation}
\mathrm{KL}(P_x \Vert Q_x)
=
\sum_{z \in \mathcal{Z}}
P_{z| x}
\mathrm{KL}\!\left(P_{Y| x,z} \Vert Q_{Y| x,z}\right)
+
\mathrm{KL}\!\left(P_{Z| x} \Vert Q_{Z| x}\right).
\end{equation}
Substituting this decomposition yields
\begin{equation}
\mathbb{E}_P[\tau_\delta]
\ge
\frac{\mathrm{kl}(\delta,1-\delta)}
{\sum_x w_x
[
\sum_z P_{z| x}
\mathrm{KL}(P_{Y| x,z} \Vert Q_{Y| x,z})
+
\mathrm{KL}(P_{Z| x} \Vert Q_{Z| x})
]}.
\end{equation}

Since the above bound holds for all $Q \in \mathrm{Alt}(P)$, we may take the
infimum over $Q$, and since the strategy is arbitrary, we may take the supremum
over $w \in \Sigma_K$. Thus,
\begin{equation}
\mathbb{E}_P[\tau_\delta]
\ge
\frac{\mathrm{kl}(\delta,1-\delta)}
{\sup_{w \in \Sigma_K}
\inf_{Q \in \mathrm{Alt}(P)}
\sum_x w_x \mathrm{KL}(P_x \Vert Q_x)}.
\end{equation}
Finally, using the inequality \cite{DBLP:journals/jmlr/KaufmannCG16}
of $\mathrm{kl}(\delta,1-\delta) \ge \log(1/(2.4\delta)) = \log(1/\delta) + \Theta(1)$ completes the proof.
\end{proof}

\subsection{Proof of Proposition~\ref{prop:three_arm_reduction}}
\label{app:three_arm_red_proof}
\Propositionone*
\begin{proof}
The constraint depends only on
$Q_{x_0}, Q_{x^\star(P)}, Q_{x'}$.
For any $a\notin\{x_0,x^\star(P),x'\}$,
the term $\mathrm{KL}(P_a\Vert Q_a)$
is uniquely minimized at $Q_a=P_a$
by strict convexity of KL divergence.
Any deviation from $P_a$ strictly increases the objective
while leaving the constraint unchanged,
and hence cannot be optimal.
\end{proof}

\subsection{Proof of Proposition~\ref{prop:biconvex}}
\label{app:biconvex_proof}
\Propositiontwo*
\begin{proof}
For fixed $q_0$, the objective is a sum of KL divergences
in $r^\star$ and $r'$,
which are convex functions.
The constraint \eqref{eq:bilinear_constraint}
is affine in $(r^\star,r')$ when $q_0$ is fixed.
Similarly, for fixed $(r^\star,r')$,
the objective is convex in $q_0$,
and the constraint is affine in $q_0$.
Joint convexity fails due to bilinearity of the constraint.
\end{proof}

\subsection{Proofs for Section~\ref{sec:analysis}}
\label{app:proofs_analysis}

Throughout, fix a true model $P\in\mathcal P$.
Recall that for each arm $x\in\mathcal X$, pulling $x$ yields i.i.d.\ samples
$(Z_t,Y_t)\sim P_x$, and that the algorithm is adaptive but non-anticipating.
We use $\mathbb P_P$ and $\mathbb E_P$ for probability and expectation under $P$.

We also recall the empirical joint model $\hat P(t)=(\hat P_x(t))_{x\in\mathcal X}$,
where $\hat P_x(t)$ is the empirical distribution of $(Z,Y)$ under arm $x$ up to time $t$.
For $Q\in\mathcal P$, define the negative log-likelihood ratio
\begin{equation}
\mathcal L_t(Q)
=
\sum_{x\in\mathcal X} N_x(t)\,\mathrm{KL}(\hat P_x(t)\Vert Q_x),
\end{equation}
which is equivalent to the decomposed form used in \eqref{eq:glrt_stat}.
The GLRT statistic is
\begin{equation}
Z_t(x,y)
=
\inf_{Q\in\mathcal P:\ \theta_Q(x)\ge \theta_Q(y)} \mathcal L_t(Q).
\end{equation}

\subsubsection{Proof of Theorem~\ref{thm:delta_correct} ($\delta$-correctness)}
\label{app:delta_correctness_proof}
\Theoremthree*
\begin{proof}
For each $t\ge 1$, define the likelihood-based confidence set
\begin{equation}
\mathcal U_t
:=
\left\{
Q\in\mathcal P:\ \mathcal L_t(Q)\le \beta(t,\delta)
\right\}.
\end{equation}

\paragraph{Step 1: time-uniform containment of the true model.}
By the uniform concentration inequality stated and proven in
Appendix~\ref{app:stopping}, the threshold $\beta(t,\delta)$ can be chosen so that
\begin{equation}
\label{eq:true_in_conf_all_t}
\mathbb P_P\!\left(P\in\mathcal U_t\ \text{for all } t\ge 1\right)\ge 1-\delta.
\end{equation}
Let $\mathcal E := \{P\in\mathcal U_t\ \forall t\ge 1\}$ denote this event.

\paragraph{Step 2: stopping implies no alternative remains in $\mathcal U_t$.}
Fix $t\ge 1$ and an arm $x\neq \hat x(t)$.
By definition,
\begin{equation}
Z_t(x,\hat x(t))
=
\inf_{Q:\ \theta_Q(x)\ge \theta_Q(\hat x(t))} \mathcal L_t(Q).
\end{equation}
Hence the condition $Z_t(x,\hat x(t))\ge \beta(t,\delta)$ implies that
for every $Q$ such that $\theta_Q(x)\ge\theta_Q(\hat x(t))$ we have
$\mathcal L_t(Q)\ge \beta(t,\delta)$, i.e.,
no such $Q$ lies in $\mathcal U_t$.
Equivalently,
\begin{equation}
\label{eq:conf_set_ranks_hatx}
\forall Q\in\mathcal U_t:\quad \theta_Q(\hat x(t)) \ge \theta_Q(x).
\end{equation}

At the stopping time $\tau_\delta$, the algorithm stops only if
\begin{equation}
\min_{x\neq \hat x(\tau_\delta)} Z_{\tau_\delta}(x,\hat x(\tau_\delta))
\ge \beta(\tau_\delta,\delta),
\end{equation}
so \eqref{eq:conf_set_ranks_hatx} holds for all competitors $x\neq \hat x(\tau_\delta)$.
Thus, for all $Q\in\mathcal U_{\tau_\delta}$,
\begin{equation}
\theta_Q(\hat x(\tau_\delta)) \ge \theta_Q(x)\quad \forall x\neq \hat x(\tau_\delta),
\end{equation}
which means $\hat x(\tau_\delta) \in \arg\max_{x\in\mathcal X} \theta_Q(x)$ for every $Q\in\mathcal U_{\tau_\delta}$ if $\tau_\delta < \infty$.

\paragraph{Step 3: correctness on the event $\mathcal E$.}
On the event $\mathcal E$, we have $P\in \mathcal U_{\tau_\delta}$ (since $\tau_\delta\ge 1$),
and by Assumption~\ref{ass:unique_opt}, the conclusion above gives
$\hat x(\tau_\delta)=x^\star(P)$ on $\mathcal E$.
Therefore,
\begin{equation}
\mathbb P_P(\hat x(\tau_\delta)\neq x^\star(P), \tau_\delta < \infty)
\le \mathbb P_P(\mathcal E^c)
\le \delta,
\end{equation}
using \eqref{eq:true_in_conf_all_t}. This proves $\delta$-correctness.
\end{proof}

\subsubsection{Proof of Lemma~\ref{lem:all_cells_diverge}}
\label{app:all_cells_diverge_full_proof}
\Lemmaone*
\begin{proof}
Fix an arm $x\in\mathcal X$. We first show that $N_x(t)\to\infty$ almost surely.

\paragraph{Step 1: $x$ enters the relevant set infinitely often.}
Let $C_x(t)=\sum_{s=1}^t \mathbf 1\{x'(s)=x\}$ be the competitor counter.
By construction of the competitor-coverage rule, for every $t$,
\begin{equation}
\min_{u\in\mathcal X} C_u(t)\ \ge\ h(t)-1.
\end{equation}
In particular, $C_x(t)\ge h(t)-1$ for all sufficiently large $t$, hence $C_x(t)\to\infty$.
Moreover, whenever $x'(s)=x$ we have $x\in \mathcal A_{\mathrm{rel}}(s)$.
Therefore, $x$ belongs to $\mathcal A_{\mathrm{rel}}(t)$ infinitely often.

\paragraph{Step 2: contradiction argument for $N_x(t)$.}
Suppose by contradiction that $N_x(t)$ is bounded on some event $E$ with
$\mathbb P_P(E)>0$. Then there exists $B<\infty$ such that on $E$,
\begin{equation}
    N_x(t)\le B\ \ \forall t,
    \qquad\text{and hence}\qquad
    m_x(t):=\min_{z\in\mathcal Z}N_{x,z}(t)\le N_x(t)\le B\ \ \forall t.
\end{equation}

Since $g(t)\to\infty$ and is nondecreasing, there exists $T$ such that
$g(t)>B$ for all $t\ge T$. Thus, on $E$, for every $t\ge T$ with
$x\in\mathcal A_{\mathrm{rel}}(t)$ we have $m_x(t)<g(t)$, i.e., $x\in\mathcal D(t)$.

Let
\begin{equation}
    \mathcal T := \{t\ge T:\ x\in\mathcal A_{\mathrm{rel}}(t)\}.
\end{equation}

By Step~1, $\mathcal T$ is infinite. For each $t\in\mathcal T$, the set
$\mathcal D(t)$ is nonempty and the forcing rule selects
\begin{equation}
    X_{t+1}\in \arg\min_{u\in\mathcal D(t)} \big(m_u(t),\,N_u(t)\big).
\end{equation}

Assume that on $E$ the arm $x$ is selected only finitely many times.
Then there exists $T'\ge T$ such that for all $t\in\mathcal T$ with $t\ge T'$,
the forced choice satisfies $X_{t+1}\neq x$. Since only finitely many arms exist,
there must be an arm $u^\star\neq x$ that is selected infinitely often among these
forcing rounds.

By Assumption~\ref{ass:uniform_interior}, $P_{u^\star}(z)>0$ for all $z$, and hence,
conditional on selecting arm $u^\star$ infinitely often, the strong law yields
$N_{u^\star,z}(t)\to\infty$ for every $z$, so in particular $m_{u^\star}(t)\to\infty$
almost surely. Therefore, on $E$ we have $m_{u^\star}(t)>B$ for all sufficiently large $t$.

But for all $t$ we also have $m_x(t)\le B$. Hence, for all sufficiently large
$t\in\mathcal T$, the arm $x$ satisfies $m_x(t)<m_{u^\star}(t)$ and belongs to
$\mathcal D(t)$, so $x$ must be chosen by the minimization of $m_u(t)$ in the forcing rule,
a contradiction. We conclude that $N_x(t)\to\infty$ almost surely.

\paragraph{Step 3: divergence of all cell counts.}
Fix $z\in\mathcal Z$. Consider the subsequence of rounds at which $X_s=x$.
Conditional on these pull times, the mediator observations are i.i.d.\ with law
$P_x(\cdot)$. By the strong law of large numbers applied to
$\mathbf 1\{Z=z\}$ along this subsequence,
\begin{equation}
\frac{N_{x,z}(t)}{N_x(t)} \to P_x(z)
\quad\text{almost surely.}
\end{equation}
By Assumption~\ref{ass:uniform_interior}, $P_x(z)>0$, and since $N_x(t)\to\infty$,
it follows that $N_{x,z}(t)\to\infty$ almost surely.

Since $x$ and $z$ were arbitrary, the claim holds for all $(x,z)$.
\end{proof}

\subsubsection{Proof of Lemma~\ref{lem:plugin_consistency}}
\label{app:plugin_consistency_full_proof}

\Lemmatwo*
\begin{proof}
Define for any model $R\in\mathcal P$ and $w\in\Sigma_K$,
\begin{equation}
\Phi_R(w)
:=
\inf_{Q\in\mathrm{Alt}(R)}
\sum_{x\in\mathcal X} w_x\,\mathrm{KL}(R_x\Vert Q_x),
\qquad
\mathrm{Alt}(R):=\{Q\in\mathcal P:\ x^\star(Q)\neq x^\star(R)\}.
\end{equation}
By definition, $\hat w(t)\in\argmax_{w\in\Sigma_K}\Phi_{\hat P(t)}(w)$.

\paragraph{Step 1: Stabilization of the plug-in best arm.}
Since $x^\star(P)$ is unique, there exists a gap
\begin{equation}
\Delta
:=
\theta_P(x^\star(P))-\max_{x\neq x^\star(P)}\theta_P(x)
\;>\;0,
\end{equation}
where $\theta_P(x)$ denotes the objective value of arm $x$ under model $P$.
Because $\theta_R(x)$ is a continuous function of $R$ on $\mathcal P$
(it is a finite sum/product of probabilities on a finite alphabet),
$\hat P(t)\to P$ almost surely implies $\theta_{\hat P(t)}(x)\to\theta_P(x)$
for every $x$.
Hence, on an event of probability one, there exists a (random) time $T_0$
such that for all $t\ge T_0$,
\begin{equation}
\theta_{\hat P(t)}(x^\star(P))
>
\max_{x\neq x^\star(P)} \theta_{\hat P(t)}(x),
\end{equation}
which implies $\hat x(t)=x^\star(\hat P(t))=x^\star(P)$ for all $t\ge T_0$.

\paragraph{Step 2: For large $t$, the alternative set becomes fixed.}
On the same event, for all $t\ge T_0$ we have
\begin{equation}
\mathrm{Alt}(\hat P(t))
=
\{Q\in\mathcal P:\ x^\star(Q)\neq \hat x(t)\}
=
\{Q\in\mathcal P:\ x^\star(Q)\neq x^\star(P)\}
=: \mathrm{Alt}^\star,
\end{equation}
which no longer depends on $t$.

We now work on $t\ge T_0$ and treat $\mathrm{Alt}^\star$ as fixed.

\paragraph{Step 3: Compactness of $\mathrm{Alt}^\star$.}
Since $\mathcal P$ is a product of simplices over a finite alphabet, it is
compact. Moreover, for each $x\neq x^\star(P)$, the set
\begin{equation}
\mathcal C_x:=\{Q\in\mathcal P:\ \theta_Q(x)\ge \theta_Q(x^\star(P))\}
\end{equation}
is closed because $\theta_Q(\cdot)$ is continuous in $Q$. Therefore
\begin{equation}
\mathrm{Alt}^\star
=
\bigcup_{x\neq x^\star(P)} \mathcal C_x
\end{equation}
is a finite union of closed sets, hence closed, and thus compact as a closed
subset of compact $\mathcal P$.

\paragraph{Step 4: Continuity of the value function on a fixed alternative set.}
Fix $w\in\Sigma_K$ and define
\begin{equation}
G(R,w,Q):=\sum_{x\in\mathcal X} w_x\,\mathrm{KL}(R_x\Vert Q_x),
\qquad Q\in\mathrm{Alt}^\star.
\end{equation}
Under Assumption~\ref{ass:uniform_interior}, all probabilities involved are
bounded away from $0$ and $1$ uniformly over $R,Q\in\mathcal P$.
Hence $\mathrm{KL}(R_x\Vert Q_x)$ is finite and continuous in $(R,Q)$ for each $x$,
and therefore $G$ is continuous in $(R,w,Q)$ on
$\mathcal P\times\Sigma_K\times \mathrm{Alt}^\star$.
Since $\mathrm{Alt}^\star$ is compact, the value function
\begin{equation}
\widetilde\Phi_R(w):=\inf_{Q\in\mathrm{Alt}^\star} G(R,w,Q)
\end{equation}
is continuous in $(R,w)$ by Berge's maximum theorem.

\paragraph{Step 5: Continuity of the argmax and convergence of $\hat w(t)$.}
Define $f(R,w):=\widetilde\Phi_R(w)$ on $\mathcal P\times\Sigma_K$.
We have shown $f$ is continuous and $\Sigma_K$ is compact.
Let $w^*(P)$ denote the unique maximizer of $w\mapsto f(P,w)$.

Consider any sequence $t_n\to\infty$ and write $R_n:=\hat P(t_n)$ and
$w_n:=\hat w(t_n)$. By compactness of $\Sigma_K$, $(w_n)$ has a convergent
subsequence (not relabeled) with limit $\bar w\in\Sigma_K$.
On the almost sure event where $\hat P(t)\to P$ and $t_n\ge T_0$ eventually,
we have $R_n\to P$ and $\mathrm{Alt}(\hat P(t_n))=\mathrm{Alt}^\star$ for all
large $n$. Since $w_n$ maximizes $f(R_n,\cdot)$,
\begin{equation}
f(R_n,w_n)\ge f(R_n,w)\quad\text{for all }w\in\Sigma_K.
\end{equation}
Taking limits along the convergent subsequence and using continuity of $f$ gives
\begin{equation}
f(P,\bar w)\ge f(P,w)\quad\text{for all }w\in\Sigma_K,
\end{equation}
so $\bar w\in\argmax_{w\in\Sigma_K} f(P,w)=\{w^*(P)\}$ by uniqueness.
Hence $\bar w=w^*(P)$.

Because every convergent subsequence of $(\hat w(t))$ has the same limit
$w^*(P)$, it follows that $\hat w(t)\to w^*(P)$ almost surely.
\end{proof}

\subsubsection{Proof of Lemma~\ref{lem:forced_sublinear}}
\label{app:forced_sublinear_proof}
\Lemmathree*
\begin{proof}
Write $F(t)=F_{\mathrm{cov}}(t)+F_{\mathrm{cell}}(t)$, where
$F_{\mathrm{cov}}(t)$ (resp.\ $F_{\mathrm{cell}}(t)$) counts the rounds up to $t$
on which competitor coverage (resp.\ cell-level forcing) is active.

\paragraph{Step 1: competitor coverage is deterministically $O(h(t))$.}
Recall $C_x(t)=\sum_{s=1}^t\mathbf 1\{x'(s)=x\}$ and the coverage rule that,
whenever $\min_x C_x(t-1)<h(t)$, selects $x'(t)$ among arms with
$C_x(t-1)<h(t)$ and increments only that counter.
Thus, for each arm $x$, the rule can select $x$ under the condition
$C_x(t-1)<h(t)$ at most $h(t)$ times up to time $t$.
Summing over $K:=|\mathcal X|$ arms yields the deterministic bound
\begin{equation}
F_{\mathrm{cov}}(t)\le \sum_{x\in\mathcal X} C_x(t)\,\mathbf 1\{C_x(t)<h(t)+1\}
\le K\,h(t)+K,
\end{equation}
hence $F_{\mathrm{cov}}(t)=O(h(t))$ and therefore $F_{\mathrm{cov}}(t)=o(t)$
since $a\in(0,1)$.

\paragraph{Step 2: cell-level forcing is $O(g(t))$ almost surely.}
Fix an arm $x\in\mathcal X$ and define the cell counts
$N_{x,z}(t)=\sum_{s\le t}\mathbf 1\{X_s=x,Z_s=z\}$ and
$N_x(t)=\sum_{z}N_{x,z}(t)$.
Let
\begin{equation}
p_{\min}:=\min_{x\in\mathcal X,\;z\in\mathcal Z} P_x(z) \in (0,1)
\end{equation}
which exists and is strictly positive by Assumption~\ref{ass:uniform_interior}.
Consider the event
\begin{equation}
\mathcal E_x:=\{\forall z\in\mathcal Z:\;
\lim_{n\to\infty}\frac{N_{x,z}(t_n)}{n}=P_x(z)\ \text{along the pull times of $x$}\}.
\end{equation}
By the strong law of large numbers applied to the i.i.d.\ mediator draws under
repeated pulls of $x$, we have $\mathbb P(\mathcal E_x)=1$.
On $\mathcal E_x$, there exists an (a.s.\ finite) random time $T_x$ such that for
all $t\ge T_x$ and all $z\in\mathcal Z$,
\begin{equation}
N_{x,z}(t)\ \ge\ \tfrac{p_{\min}}{2}\,N_x(t).
\tag{$\star$}
\label{eq:cell_lb}
\end{equation}
Consequently, for $t\ge T_x$,
\begin{equation}
m_x(t):=\min_{z\in\mathcal Z}N_{x,z}(t)
\ \ge\ \tfrac{p_{\min}}{2}\,N_x(t).
\end{equation}
Therefore, if $x$ is cell-deficient at time $t\ge T_x$, i.e.\ $m_x(t)<g(t)$, then
necessarily
\begin{equation}
N_x(t)\ <\ \frac{2}{p_{\min}}\,g(t).
\tag{$\dagger$}
\label{eq:Nx_needed}
\end{equation}
Now let $F_{\mathrm{cell},x}(t)$ be the number of times up to $t$ that the
algorithm selects arm $x$ \emph{because of cell-level forcing}.
Trivially, $F_{\mathrm{cell},x}(t)\le N_x(t)$.
Moreover, once \eqref{eq:Nx_needed} fails (i.e.\ once
$N_x(t)\ge \tfrac{2}{p_{\min}}g(t)$) and $t\ge T_x$, the arm $x$ cannot be
cell-deficient anymore, hence it cannot be selected due to cell forcing.
Thus, on $\mathcal E_x$ and for all large enough $t$,
\begin{equation}
F_{\mathrm{cell},x}(t)
\ \le\ \frac{2}{p_{\min}}\,g(t)+T_x.
\end{equation}
Summing over all $x\in\mathcal X$ and using finiteness of $\mathcal X$ gives,
on $\bigcap_x\mathcal E_x$ (an event of probability one),
\begin{equation}
F_{\mathrm{cell}}(t)
=\sum_{x\in\mathcal X}F_{\mathrm{cell},x}(t)
\ \le\ \frac{2K}{p_{\min}}\,g(t)+\sum_{x\in\mathcal X}T_x
\ =\ O(g(t))\qquad\text{a.s.}
\end{equation}
Since $g(t)=o(h(t))$ and $h(t)=t^a$ with $a\in(0,1)$, we have $g(t)=o(t)$ and hence
$F_{\mathrm{cell}}(t)=o(t)$ almost surely.

\paragraph{Step 3: conclusion.}
Combining the bounds yields, almost surely,
\begin{equation}
F(t)=F_{\mathrm{cov}}(t)+F_{\mathrm{cell}}(t)
=O(h(t))+O(g(t))=o(t).
\end{equation}
\end{proof}

\subsubsection{Proof of Lemma~\ref{lem:tracking_final}}
\label{app:tracking_sublinear_proof}
\Lemmafour*
\begin{proof}
We work on the almost sure event on which both
\begin{equation}
\hat w(t)\to w^*(P)
\qquad\text{and}\qquad
F(t)=o(t)
\end{equation}
hold, where $F(t)$ denotes the number of forced rounds up to time $t$.

Let $\mathcal F\subseteq\mathbb N$ be the (random) set of forced rounds and
define $n(t):=t-F(t)$, the number of non-forced rounds up to time $t$.
Then $n(t)\to\infty$ and $n(t)/t\to 1$.

Let $\tau_1<\tau_2<\dots$ denote the (decision) times at which
the arm selection rule uses D-tracking rather than forcing.
Since at time $t$ the algorithm chooses $X_{t+1}$,
the non-forced pull counts are defined by
\begin{equation}
\widetilde N_x(k):=\sum_{j=1}^k \mathbf 1\{X_{\tau_j+1}=x\}.
\end{equation}
We have the decomposition
\begin{equation}
N_x(t)
=
\widetilde N_x(n(t))
+
N_x^{\mathrm{F}}(t),
\qquad
0\le N_x^{\mathrm{F}}(t)\le F(t).
\label{eq:tracking_decomposition}
\end{equation}

\paragraph{Step 1: induced dynamics on non-forced rounds.}
On each non-forced decision time $\tau_k$, the algorithm selects
\begin{equation}
X_{\tau_k+1}
\in
\arg\max_{x\in\mathcal X}
(\tau_k \hat w_x(\tau_k)-N_x(\tau_k)).
\end{equation}
Using \eqref{eq:tracking_decomposition},
\begin{equation}
\tau_k \hat w_x(\tau_k)-N_x(\tau_k)
=
k\,\hat w_x(\tau_k)-\widetilde N_x(k)
+
R_{k,x},
\end{equation}
where
\begin{equation}
R_{k,x}
=
(\tau_k-k)\hat w_x(\tau_k)-N_x^{\mathrm{F}}(\tau_k).
\end{equation}
Since $\tau_k-k=F(\tau_k)=o(\tau_k)=o(k)$ and
$0\le N_x^{\mathrm{F}}(\tau_k)\le F(\tau_k)=o(k)$,
we have uniformly in $x$,
\begin{equation}
\sup_{x\in\mathcal X}|R_{k,x}|=o(k).
\end{equation}
Thus the update rule on non-forced rounds is an $o(k)$ perturbation
of the ideal D-tracking rule
\begin{equation}
\arg\max_x(k\,\hat w_x(\tau_k)-\widetilde N_x(k)).
\end{equation}

\paragraph{Step 2: convergence along non-forced rounds.}
Since $\hat w(\tau_k)\to w^*(P)$, the deterministic stability argument
for D-tracking (cf.\ \citep[Appendix~B.2--B.3]{garivier2016optimalbestarmidentification})
yields
\begin{equation}
\frac{\widetilde N_x(k)}{k}\to w_x^*(P)
\qquad\text{for all }x\in\mathcal X.
\end{equation}

\paragraph{Step 3: lifting back to real time.}
Using \eqref{eq:tracking_decomposition},
\begin{equation}
\frac{N_x(t)}{t}
=
\frac{n(t)}{t}
\cdot
\frac{\widetilde N_x(n(t))}{n(t)}
+
\frac{N_x^{\mathrm{F}}(t)}{t}.
\end{equation}
Since $n(t)/t\to 1$, $\widetilde N_x(n)/n\to w_x^*(P)$,
and $N_x^{\mathrm{F}}(t)/t\le F(t)/t\to 0$,
we conclude
\begin{equation}
\frac{N_x(t)}{t}\to w_x^*(P)
\qquad\text{for all }x\in\mathcal X.
\end{equation}
\end{proof}

\subsubsection{Proof of Lemma~\ref{lem:glrt_growth}}
\label{app:glrt_growth_proof}
\Lemmafive*
\begin{proof}
Work on an event of probability one on which the following hold simultaneously:
(i) $\hat P(t)\to P$,
(ii) $N_a(t)/t\to w_a^*(P)$ for all $a\in\mathcal X$, and
(iii) $\hat x(t)=x^\star(P)$ for all sufficiently large $t$.
All three were established in Section~\ref{sec:analysis}.

Fix a competitor $x\neq x^\star(P)$.
For all sufficiently large $t$, we have $\hat x(t)=x^\star(P)$, and therefore
\begin{equation}
Z_t(x,\hat x(t))
=
\inf_{\substack{Q\in\mathcal P:\\ \theta_Q(x)\ge \theta_Q(x^\star(P))}}
\sum_{a\in\mathcal X}
N_a(t)\,\mathrm{KL}(\hat P_a(t)\Vert Q_a).
\end{equation}

Let $(t_n)$ be a subsequence with $t_n\to\infty$.
For each $n$, choose $Q_n$ in the constraint set such that
\begin{equation}
Z_{t_n}(x,\hat x(t_n))
\ge
\sum_{a\in\mathcal X}
N_a(t_n)\,\mathrm{KL}(\hat P_a(t_n)\Vert (Q_n)_a)
-\frac{1}{n}.
\end{equation}
Since $\mathcal P$ is compact, we may extract a subsequence (not relabeled)
such that $Q_n\to Q_\infty\in\mathcal P$.

Because $\hat P(t_n)\to P$ and $\theta_Q(\cdot)$ is continuous in $Q$,
the constraint $\theta_{Q_n}(x)\ge \theta_{Q_n}(x^\star(P))$
passes to the limit, hence $Q_\infty$ satisfies
$\theta_{Q_\infty}(x)\ge \theta_{Q_\infty}(x^\star(P))$.

Under Assumption~\ref{ass:uniform_interior},
$\mathrm{KL}(\cdot\Vert\cdot)$ is continuous on $\mathcal P\times\mathcal P$.
Thus for each arm $a$,
\begin{equation}
\mathrm{KL}(\hat P_a(t_n)\Vert (Q_n)_a)
\to
\mathrm{KL}(P_a\Vert (Q_\infty)_a).
\end{equation}
Combining with $N_a(t_n)/t_n\to w_a^*(P)$ yields
\begin{equation}
\lim_{n\to\infty}
\frac{1}{t_n}
\sum_{a\in\mathcal X}
N_a(t_n)\,\mathrm{KL}(\hat P_a(t_n)\Vert (Q_n)_a)
=
\sum_{a\in\mathcal X}
w_a^*(P)\,\mathrm{KL}(P_a\Vert (Q_\infty)_a).
\end{equation}
Therefore,
\begin{equation}
\liminf_{n\to\infty}
\frac{1}{t_n}
Z_{t_n}(x,\hat x(t_n))
\ge
\sum_{a\in\mathcal X}
w_a^*(P)\,\mathrm{KL}(P_a\Vert (Q_\infty)_a)
\ge
\inf_{\substack{Q\in\mathcal P:\\
\theta_Q(x)\ge \theta_Q(x^\star(P))}}
\sum_{a\in\mathcal X}
w_a^*(P)\,\mathrm{KL}(P_a\Vert Q_a).
\end{equation}

Since the sequence $(t_n)$ was arbitrary and limit infimum corresponds to one of such subsequences, we conclude
\begin{equation}
\liminf_{t\to\infty}
\frac{1}{t}
Z_t(x,\hat x(t))
\ge
\inf_{\substack{Q\in\mathcal P:\\ \theta_Q(x)\ge \theta_Q(x^\star(P))}}
\sum_{a\in\mathcal X}
w_a^*(P)\,\mathrm{KL}(P_a\Vert Q_a)
\quad\text{w.p. 1}.
\end{equation}

Taking the minimum over $x\neq x^\star(P)$ and recalling the definition of
$T^*(P)$ from Theorem~\ref{thm:nde_lower_bound} yields
\begin{equation}
\liminf_{t\to\infty}
\frac{1}{t}
\min_{x\neq x^\star(P)} Z_t(x,\hat x(t))
\ge
\frac{1}{T^*(P)}.
\end{equation}
\end{proof}

\subsubsection{Proof of Theorem \ref{thm:asympt_opt_as} (Almost-sure asymptotic optimality)}
\label{app:asympt_opt_as_proof}
\Theoremfive*
\begin{proof}
Let
\begin{equation}
Z(t):=\min_{x\neq \hat x(t)} Z_t(x,\hat x(t))
\end{equation}
be the GLRT statistic used in the stopping rule
\begin{equation}
\tau_\delta=\inf\{t\ge 1:\; Z(t)\ge \beta(t,\delta)\}.
\end{equation}

\paragraph{Step 1: Linear growth of the GLRT.}
By Lemma~\ref{lem:glrt_growth},
\begin{equation}
\liminf_{t\to\infty}\frac{Z(t)}{t}\ge \frac{1}{T^\star(P)}
\qquad\text{w.p. 1}
\end{equation}
Hence, for every $\varepsilon>0$, there exists a (random) finite time $t_\varepsilon$ (w.p.\ 1)
such that for all $t\ge t_\varepsilon$,
\begin{equation}
Z(t)\ge \frac{t}{(1+\varepsilon)T^\star(P)}
\qquad\text{w.p. 1}
\label{eq:Z_linear_as}
\end{equation}

\paragraph{Step 2: Dominating $\beta(t,\delta)$ by a polynomial.}
From \eqref{eq:beta}
\begin{equation}
\beta(t,\delta)
=
\log\Big(\tfrac{\pi^2 t^{2+2K M_Z}}{\delta}\Big)
+
6 M_Z K \log(t+1).
\end{equation}
For all $t\ge 1$, using $\log(t+1)\le \log(2t)=\log 2 + \log t$, we obtain
\begin{align}
\beta(t,\delta)
&\le
\log(1/\delta)
+
\log(\pi^2)
+
(2+2KM_Z)\log t
+
6KM_Z(\log 2+\log t)
\nonumber\\
&=
\log\!\Big(\frac{C\, t^{\alpha}}{\delta}\Big),
\label{eq:beta_poly_dom}
\end{align}
where
\begin{equation}
\alpha:=2+8KM_Z,
\qquad
C:=\pi^2\,2^{6KM_Z}.
\end{equation}

\paragraph{Step 3: Solving the implicit inequality.}
Combining \eqref{eq:Z_linear_as} and \eqref{eq:beta_poly_dom}, for all $t\ge t_\varepsilon$,
the condition $Z(t)\ge \beta(t,\delta)$ is implied by
\begin{equation}
\frac{t}{(1+\varepsilon)T^\star(P)}
\ge
\log\!\Big(\frac{C t^\alpha}{\delta}\Big).
\end{equation}
Let $a:=(1+\varepsilon)T^\star(P)$. Then it suffices that
\begin{equation}
t \ge a \log\!\Big(\frac{C t^\alpha}{\delta}\Big).
\end{equation}
Applying Lemma~18 of \citet{garivier2016optimalbestarmidentification} yields that,
for all sufficiently small $\delta$,
\begin{equation}
\tau_\delta
\le
a\Big(
\log(C/\delta)
+
\alpha \log\log(C/\delta)
+
O(1)
\Big)
\qquad\text{w.p. 1}
\label{eq:tau_upper_loglog}
\end{equation}

\paragraph{Step 4: Limsup bound.}
Dividing \eqref{eq:tau_upper_loglog} by $\mathrm{kl}(\delta,1-\delta)$ and letting $\delta\downarrow 0$
gives, w.p. 1,
\begin{equation}
    \limsup_{\delta\rightarrow 0}\frac{\tau_\delta}{\mathrm{kl}(\delta,1-\delta)}
    \le
    (1+\varepsilon)T^\star(P),
\end{equation}

since $\mathrm{kl}(\delta,1-\delta)=\log(1/\delta)+O(1)$ and $\log\log(1/\delta)=o(\mathrm{kl}(\delta,1-\delta))$.
As $\varepsilon>0$ is arbitrary, letting $\varepsilon\downarrow 0$ yields
\begin{equation}
  \limsup_{\delta\rightarrow 0}\frac{\tau_\delta}{\mathrm{kl}(\delta,1-\delta)}
    \le
    T^\star(P)
    \qquad\text{w.p.\ 1}. 
\end{equation}

\end{proof}

\end{document}